\documentclass[a4paper,12pt,reqno]{amsart} 
\usepackage[utf8]{inputenc}
\usepackage{amsaddr}
\usepackage{tabularx}
\usepackage{booktabs} 
\usepackage{graphicx}
\graphicspath{ {images/} }
\usepackage{amssymb}
\usepackage{graphicx}
\usepackage{tikz-cd}
\usepackage{tikz}
 \usetikzlibrary{decorations.pathmorphing}
\usepackage{amsmath}
\usepackage{dsfont}
\usepackage{shuffle}
\usepackage{afterpage}
\usepackage{setspace}
\usepackage[utf8]{inputenc}

\usepackage{mathrsfs}
\usepackage{mathbbol}
\usepackage[margin=1in]{geometry}
\usepackage{tabularx,booktabs}
\usepackage{url}
\usepackage{fullpage}
\usepackage{graphicx}
\newcommand{\junk}[1]{}
\usepackage[parfill]{parskip}
\usepackage{amsmath, amssymb}
\usepackage{booktabs}
\usepackage[makeroom]{cancel}
\usepackage{nicefrac}
\usepackage{listings}
\usepackage{float}
\usepackage{wrapfig}
\usepackage{afterpage}
\usepackage{enumerate}
\usepackage{enumitem}
\usepackage{tabularx}
\usepackage{adjustbox}
\usepackage{comment}
\usepackage{amsmath}
\usepackage{tabularx}
\usepackage{adjustbox}
\usepackage{blindtext}
\usepackage{microtype}
\usepackage{graphicx}
\usepackage{subfigure}
\usepackage{multirow}
\usepackage{booktabs} 
\usepackage{placeins}
\usepackage[
]{caption}
\usepackage{blindtext}
\usepackage[utf8]{inputenc}
\usepackage{makecell}
\usepackage{tabularx,booktabs}
\usepackage{eqparbox} 
\usepackage{algorithm}
\usepackage{algorithmic} 
\usepackage{tabularx}
\usepackage{url}
\usepackage{physics}
\usepackage{tikz}
\usetikzlibrary{positioning, fit, arrows.meta}
\newcommand{\revised}[1]{{\color{black}{#1}}}
\newcommand{\revisedd}[1]{{\color{black}{#1}}}
\usetikzlibrary{decorations.pathreplacing, arrows.meta} 

\usetikzlibrary{positioning, calc}
\tikzset{
  tens/.style={draw, fill=gray!20, minimum size=8mm, inner sep=1pt},
  leg/.style={thick},
  vleg/.style={thick},
}

\tikzset{
  tens/.style={draw, fill=gray!20, rectangle, rounded corners=2pt, minimum width=9mm, minimum height=7mm, inner sep=2pt},
  tensbra/.style={draw, fill=gray!10, rectangle, rounded corners=2pt, minimum width=9mm, minimum height=7mm, inner sep=2pt},
  mpo/.style={draw, fill=orange!30, rectangle, rounded corners=2pt, minimum width=9mm, minimum height=7mm, inner sep=2pt},
  phys/.style={thick},
  virt/.style={thick},
  mpoedge/.style={thick, dashed},
  conj/.style={densely dotted},
  every node/.style={font=\small}
}
\DeclareSymbolFont{largesymbolsA}{U}{txexa}{m}{n}
\DeclareMathSymbol{\varprod}{\mathop}{largesymbolsA}{16}

\newcommand{\be}[1]{\begin{equation}}
\newcommand{\ee}[1]{\end{equation}}

\newcommand{\bl}[1]{{\color{black}{}}}
\usepackage{graphicx}

\graphicspath{ {./images/} }
\newtheorem{example}{\textbf{Example}} 
\newtheorem{theorem}{\textbf{Theorem}}
\newtheorem{lemma}[theorem]{\textbf{Lemma} }
\newtheorem{proposition}[theorem]{\textbf{Proposition}} 
\newtheorem{remark}[theorem]{\textbf{Remark}}

\newtheorem{definition}[theorem]{\textbf{Definition}}

\newtheorem{assumption}[theorem]{\textbf{Assumption}}

\usepackage{algorithm}
\usepackage{algorithmic}
\makeatletter
\newcommand\fs@norules{\def\@fs@cfont{\bfseries}\let\@fs@capt\floatc@ruled
  \def\@fs@pre{}%
  \def\@fs@post{}%
  \def\@fs@mid{\kern3pt}%
  \let\@fs@iftopcapt\iftrue}
\makeatother
\floatstyle{norules}
\restylefloat{algorithm}
\usepackage{graphicx}
\usepackage{textcomp}
\usepackage{xcolor}
\def\BibTeX{{\rm B\kern-.05em{\sc i\kern-.025em b}\kern-.08em
    T\kern-.1667em\lower.7ex\hbox{E}\kern-.125emX}}

\usepackage{hyperref} 
\usepackage[all]{xy}
\usetikzlibrary{knots} 
\usepackage[lite]{amsrefs}
\usepackage{bbm}

\usepackage{tabularx}
\usepackage{booktabs} 

\newcommand{\eq}[1][r]
   {\ar@<-3pt>@{-}[#1]
    \ar@<-1pt>@{}[#1]|<{}="gauche"
    \ar@<+0pt>@{}[#1]|-{}="milieu"
    \ar@<+1pt>@{}[#1]|>{}="droite"
    \ar@/^2pt/@{-}"gauche";"milieu"
    \ar@/_2pt/@{-}"milieu";"droite"}

  \entrymodifiers={!!<0pt,0.7ex>+} 
        
\newcommand{\bigon}[4][r]{
    \ar@/^1pc/[#1]^{#2}_*=<0.3pt>{}="HAUT"
    \ar@/_1pc/[#1]_{#3}^*=<0.3pt>{}="BAS"
    \ar@{=>} "HAUT";"BAS" ^{#4}
  }

\newcommand{\bigons}[6][r]{  
    \ar@/^2pc/[#1]^{#2}_*=<0.3pt>{}="HAUT"
    \ar@{}    [#1]     ^*=<0.3pt>{}="MILIEUHAUT"
                       _*=<0.3pt>{}="MILIEUBAS"
    \ar[#1]_(0.3){#3}                  
    \ar@/_2pc/[#1]_{#4}^*=<0.3pt>{}="BAS"
    \ar@{=>} "HAUT";"MILIEUHAUT" ^{#5}
    \ar@{=>} "MILIEUBAS";"BAS" ^{#6}
  }       
\usepackage{stmaryrd} 
\theoremstyle{definition}

\theoremstyle{remark}

\AtBeginEnvironment{example}{%
  \pushQED{\qed}%
}
\AtEndEnvironment{example}{\popQED\endexample}

\allowdisplaybreaks     

\usepackage{xcolor}   
\usepackage[most]{tcolorbox} 

\newtcolorbox{myrednote}{
  colback=red!5!white,
  colframe=red!75!black,
  fonttitle=\bfseries,
  title=Note
}
\newtcolorbox{myorangenote}{
  colback=orange!5!white,
  colframe=orange!75!black,
  fonttitle=\bfseries,
  title=Note
}
\newtcolorbox{mybluenote}{
  colback=blue!5!white,
  colframe=blue!75!black,
  fonttitle=\bfseries,
  title=Note
}

\title[]{
No-Rank Tensor Decomposition Using Metric Learning}
\author{Maryam Bagherian}
\address{Department of Mathematics and Statistics, Idaho State University\\
	Physical Science Complex|  921 S. 8th Ave., Stop 8085 | Pocatello, ID 83209} 
\email{maryambagherian@isu.edu}

\begin{document}
\begin{abstract}
	Tensor decomposition of high-dimensional data often struggles to capture semantically or physically meaningful structures, particularly when relying on reconstruction objectives and fixed-rank constraints. We introduce a no-rank tensor decomposition framework based on metric learning, which replaces reconstruction objectives with a similarity-driven optimization. By combining a triplet loss with diversity and uniformity regularization, the method learns embeddings where distances naturally reflect semantic and physical relationships, supported by theoretical guarantees on convergence and metric properties.
	
	We evaluate the approach on diverse datasets, including face recognition (LFW, Olivetti), brain connectivity (ABIDE), and simulated physical systems (galaxies, crystals). In comprehensive comparisons against classical methods (PCA, t-SNE, UMAP), tensor decompositions (CP, Tucker, t-SVD), and deep learning models (VAE, DEC, transformer-based embeddings), our method produces embeddings that preserve physically and semantically relevant relationships and achieve competitive clustering performance. While transformers often excel in predictive accuracy on large datasets, our method provides interpretable embeddings and remains effective in small-data regimes where transformer training may be infeasible.
	
	This work establishes metric learning as a principled paradigm for tensor analysis, emphasizing physical interpretability and semantic relevance over pixel-level reconstruction, and offering an efficient and robust alternative in data-scarce scientific domains.
\end{abstract}

\maketitle

\section{Introduction}

Tensor decomposition and representation learning are two fundamental paradigms for extracting meaningful structure from multi-dimensional data. Traditional tensor decomposition methods, such as CANDECOMP/PARAFAC (CP) \cite{kolda2009tensor} and Tucker decomposition \cite{tucker1966some}, provide mathematically elegant approaches for factorizing tensors into interpretable components.

The CP decomposition approximates an $N$-way tensor $\mathcal{X} \in \mathbb{R}^{I_1 \times I_2 \times \cdots \times I_N}$ as a sum of $R$ rank-one tensors:
\begin{equation}
\mathcal{X} \approx \sum_{r=1}^{R} \mathbf{u}_r^{(1)} \circ \mathbf{u}_r^{(2)} \circ \cdots \circ \mathbf{u}_r^{(N)},
\end{equation}
where $\circ$ denotes the outer product, and $R$ is the \textit{rank} of the decomposition. Similarly, the Tucker decomposition employs a core tensor $\mathcal{G} \in \mathbb{R}^{R_1 \times R_2 \times \cdots \times R_N}$ and factor matrices $\mathbf{U}^{(n)} \in \mathbb{R}^{I_n \times R_n}$:
\begin{equation}
\mathcal{X} \approx \mathcal{G} \times_1 \mathbf{U}^{(1)} \times_2 \mathbf{U}^{(2)} \times_3 \cdots \times_N \mathbf{U}^{(N)},
\end{equation}
where $\times_n$ denotes the $n$-mode product.

While these methods offer strong uniqueness guarantees \cite{kruskal1977three} and interpretability, they require pre-specification of rank parameters $(R, R_1, \dots, R_N)$ and are inherently limited to multi-linear relationships.

Representation learning \cite{bengio2013representation} provides a complementary approach, learning useful data representations through neural networks without explicit structural constraints. Deep autoencoders, for instance, learn an encoding function $f_\theta$ mapping input data to latent codes $\mathbf{z}$ and a decoding function $g_\phi$ reconstructing the original input:
\begin{equation}
\mathcal{\hat{X}} = g_\phi(f_\theta(\mathcal{X})),
\end{equation}
The latent representation $\mathbf{z}$ captures essential factors of variation in the data, with the model capacity determining the effective complexity rather than pre-defined rank constraints.

Recent work has begun bridging these domains, with neural networks enhancing tensor decompositions  and implicit neural representations demonstrating remarkable compression capabilities \cite{tancik2020fourier}. However, the fundamental challenge of \textit{rank selection} persists across traditional methods.

This work introduces a ``no-rank'' paradigm where representation learning principles enable adaptive, data-driven tensor decomposition without explicit rank specification, addressing key limitations of both established approaches while leveraging their respective strengths.

The rest of the manuscript is organized as follows: Related work are discussed in Section~\ref{sec:related_work} followed by the model formulation in Section~\ref{sec:model}.
Section~\ref{sec:metric} provides an overview of the metric analysis and result interpretation and experimental results over real and simulated dataset are presented and discussed in Section~\ref{sec:res}. We conclude in Section~\ref{sec:con} by discussing the advantages,  limitations and future direction of the proposed framework.

\section{Related Work}
\label{sec:related_work}

This section reviews the relevant literature in these fields, highlighting the gap that  \textit{no-rank} metric learning framework aims to fill.

\subsection*{Traditional Tensor Decomposition Methods}

Tensor decomposition methods have long been the workhorse for analyzing multi-way data. The Canonical Polyadic (CP) decomposition \cites{hitchcock1927expression} and the Tucker decomposition \cite{tucker1966some} are the two most foundational approaches. Both methods impose a strict low-rank constraint on the data, seeking an approximation $\hat{\mathcal{X}}$ that minimizes the reconstruction error $||\mathcal{X} - \hat{\mathcal{X}}||_F^2$. While effective for data compression and denoising, this reconstructive objective is not inherently aligned with discriminative tasks like classification or clustering. The requirement to pre-define the rank or multilinear rank is a significant limitation, as the intrinsic data complexity is often unknown and may not be well-represented by a low-rank model \cite{acar2011scalable}.

Subsequent advancements, such as tensor-train decompositions \cite{oseledets2011tensor} and non-negative tensor factorizations \cite{cichocki2009nonnegative}, have improved scalability and interpretability but have largely remained within the reconstructive paradigm. These methods are fundamentally linear and struggle to capture the complex, non-linear manifolds on which high-dimensional data often resides.

Recent work has enhanced tensor decompositions by incorporating structured operators. For instance, authors in \cite{xue2024tensor} proposed a tensor convolution-like low-rank dictionary for high-dimensional image representation, improving the capture of local spatial correlations. Nevertheless, such methods remain reconstructive and rely on explicit low-rank constraints, limiting their suitability for discriminative tasks.
	
Other modern tensor decomposition approaches include t-SVD-based frameworks, which generalize matrix SVD to third-order tensors using the t-product. These methods have been applied to representation learning via tensor neural networks under transformed low-rank constraints, \cite{wang2023transformed}, and to manifold learning using t-product geometry \cite{wangtowards}. Similar to CP, Tucker, and convolutional low-rank dictionaries, t-SVD approaches rely on explicit low-rank structures and linear transformations, and are primarily reconstructive. In contrast, our \textit{no-rank} metric learning framework learns embeddings without imposing rank constraints, directly optimizing semantic similarity in a non-linear, discriminative manner.

\subsection*{Dimensionality Reduction and Manifold Learning}

To address non-linearity, a separate lineage of work focused on manifold learning and non-linear dimensionality reduction \cite{bagherian2023classical}. Techniques such as Isomap \cite{tenenbaum2000global}, Locally Linear Embedding (LLE) \cite{roweis2000nonlinear}, and Laplacian Eigenmaps \cite{belkin2003laplacian} aim to preserve geometric properties of the data. More recently, t-Distributed Stochastic Neighbor Embedding (t-SNE) \cite{maaten2008visualizing} and Uniform Manifold Approximation and Projection (UMAP) \cite{mcinnes2018umap} have become standards for visualization, excelling at preserving local neighborhood structures.

However, these methods are predominantly \textit{unsupervised} and \textit{geometry-preserving}. They lack a mechanism to incorporate supervisory signals, such as class labels, to guide the feature learning process. Consequently, the resulting low-dimensional embeddings may not optimize for class separability, which is crucial for many analysis tasks. Furthermore, they operate as separate pre-processing steps and are not integrated into an end-to-end trainable model for feature extraction.

\subsection*{Deep Metric and Representation Learning}

The advent of deep learning catalyzed a shift towards learning representations directly optimized for a specific task. A pivotal development in this space is \textit{metric learning} \cite{kulis2013metric}, which aims to learn a distance function that reflects semantic similarity. The contrastive loss \cite{hadsell2006dimensionality} and, more influentially, the triplet loss \cite{schroff2015facenet} provided a powerful framework for this. By pulling an anchor sample closer to a positive sample than to a negative sample by a margin, these losses directly optimize the embedding space for discrimination.

This paradigm has driven state-of-the-art performance in face recognition \cite{wang2018cosface}, image retrieval \cite{faghri2018vse}, and person re-identification \cite{hermans2017defense}. This work draws direct inspiration from these successes but adapts the triplet loss framework to the problem of tensor decomposition, moving beyond its typical application in computer vision.

To prevent pathological solutions like dimensional collapse, recent work has emphasized the importance of regularization. Wang and Isola \cite{wang2020understanding} identified \textit{alignment} and \textit{uniformity} as key properties of effective representations, which has led to the use of uniformity losses and diversity penalties on the embedding correlation matrix \cite{kulkarni2019canonical}. The proposed framework incorporates these insights to ensure a well-structured and effective embedding space.

\subsection*{Metric Learning for Data Analysis}

There is a growing recognition of the limitations of traditional methods for data analysis. In domains like astronomy, methods for galaxy classification have evolved from manual taxonomy \cite{hubble1926extragalactic} to machine learning approaches using hand-crafted features \cite{dieleman2015rotation} and, more recently, deep convolutional networks \cite{walmsley2022galaxy}. Similarly, in materials science, the analysis of crystal structures has been tackled with symmetry-based descriptors \cite{isayev2017universality} and graph neural networks \cite{xie2018crystal}.

However, the application of deep metric learning in tensor-structured scientific data remains nascent. While contrastive and metric learning techniques have been successfully applied in domains such as medical imaging \cite{chen2020simclr_medical} and scientific signal representation learning \cite{oord2018cpc}, a generalized framework that replaces the core principles of tensor decomposition for multiway data is still lacking. 
	Existing approaches typically treat metric learning as an auxiliary representation module rather than as a fundamental alternative to tensor decomposition.

Built on the previous works \cites{bagherian2021coupled, bagherian2022bilevel, bagherian2024tensor}, the proposed \textit{no-rank} tensor decomposition framework synthesizes ideas from these disparate fields. We reframe the problem of tensor analysis from one of \textit{reconstruction} to one of \textit{discrimination}, drawing on the power of triplet-based metric learning. Unlike traditional tensor methods, we impose no explicit rank constraints and leverage deep non-linear networks to capture complex data manifolds. Unlike unsupervised manifold learning, we directly optimize for semantic similarity using label information. And unlike standard metric learning applications, the proposed method is positioned as a direct, end-to-end replacement for tensor decomposition in the data pipeline, a contribution with practical implications.

\section{Metric Learning Framework}\label{sec:model}
\subsection{Theoretical Framework}

Traditional tensor decomposition methods, such as Canonical Polyadic (CP) and Tucker decompositions, impose strict rank constraints that may not align with the intrinsic geometry of high-dimensional data. 	Tensor rank is NP-hard and unstable to compute \cite{hillar2013most}. Therefore, we propose a \textit{no-rank} tensor decomposition framework based on metric learning, which learns data-driven similarity structures without explicit rank constraints.\\

		Given a learned embedding matrix $\mathbf{Z} \in \mathbb{R}^{n \times d}$ where rows correspond to data samples and columns to embedding dimensions, the \emph{effective rank} quantifies the number of significant, linearly independent dimensions in the learned representation. Let $\sigma_1 \geq \sigma_2 \geq \cdots \geq \sigma_d \geq 0$ be the singular values of $\mathbf{Z}$.
		The following standard measures of effective rank are equivalent:
		\begin{enumerate}[label=(\roman*)]
			\item {Explained variance threshold:} For $\epsilon \in [0,1)$,
			\[
			\text{erank}_{\epsilon}(\mathbf{Z}) = \min\left\{ k : \frac{\sum_{i=1}^k \sigma_i^2}{\sum_{i=1}^d \sigma_i^2} \geq 1 - \epsilon \right\}.
			\]
			
			\item {Spectral decay ratio:} For $\tau \in (0,1]$,
			\[
			\text{erank}_{\tau}(\mathbf{Z}) = \max\left\{ k : \sigma_k \geq \tau \cdot \sigma_1 \right\}.
			\]
			
			\item {Entropy-based effective rank:}
			\[
			\text{erank}_{\text{ent}}(\mathbf{Z}) = \exp\left(-\sum_{i=1}^d p_i \log p_i\right),
			\quad \text{where } p_i = \frac{\sigma_i^2}{\sum_{j=1}^d \sigma_j^2}.
			\]
		\end{enumerate}
		In the context of no-rank tensor decomposition, these measures capture the \emph{implicit} dimensionality learned through metric optimization. When embedding dimensions become decorrelated (a property enforced by our diversity regularization), 
		the columns of $\mathbf{Z}$ become approximately orthonormal and the nonzero singular values concentrate around one.
		Consequently, all three effective rank measures converge to the embedding dimension $d$.
\begin{definition}[No-Rank Tensor Decomposition]
	Given an $N$-th order tensor $\mathcal{X} \in \mathbb{R}^{I_1 \times \cdots \times I_N}$, 
	a \emph{no-rank tensor decomposition} learns embedding functions $f^{(n)}: \mathbb{R}^{I_{\setminus n}} \to \mathbb{R}^{d}$
	(possibly shared across modes) that map mode-$n$ fibers to semantic embeddings.
	
	For each mode $n = 1,\dots,N$, let $\mathbf{x}_{i_n}^{(n)}$ be the mode-$n$ fiber at index $i_n$,
	and define $\mathbf{z}_{i_n}^{(n)} = f^{(n)}(\mathbf{x}_{i_n}^{(n)}) \in \mathbb{R}^d$.
	
	The decomposition is characterized by (i) an implicit similarity tensor $\mathcal{S} \in \mathbb{R}^{I_1 \times \cdots \times I_N}$ with entries
		\[
		\mathcal{S}_{i_1,\dots,i_N} = \langle \mathbf{z}_{i_1}^{(1)}, \mathbf{z}_{i_2}^{(2)}, \dots, \mathbf{z}_{i_N}^{(N)} \rangle.
		\] (ii) the \emph{effective rank} of the decomposition is defined by the effective rank 
		of the embedding matrices $\mathbf{Z}^{(n)} = [\mathbf{z}_1^{(n)^\top}, \dots, \mathbf{z}_{I_n}^{(n)^\top}]^\top$,
		determined implicitly through optimization rather than prescribed a priori.
\end{definition}
	This notion of decomposition is formalized in Proposition~\ref{prop:metric_factorization}, which shows that the induced similarity tensor admits a CP-style factorization whose effective rank equals the embedding dimension.\\

Let $\mathcal{X} = \{\mathbf{x}_1, \mathbf{x}_2, \dots, \mathbf{x}_n\} \subset \mathbb{R}^D$ be a set of multi-dimensional data points with corresponding labels $\mathcal{Y} = \{y_1, y_2, \dots, y_n\}$. The goal is to learn a mapping $f: \mathbb{R}^D \to \mathbb{R}^d$, where $d \ll D$, such that the resulting low-dimensional embedding space preserves semantically meaningful relationships from the original space.

Let $\mathbf{x}_i = \text{vec}(\mathcal{X}_{:,:,\ldots,i})$ be the vectorized $i$-th mode-$N$ slice of the tensor. We learn embeddings $\mathbf{z}_i = f(\mathbf{x}_i)$ by constructing triplets $(a, p, n) \in \mathcal{T}$, where the \textit{anchor} ($a$) is a reference sample: $\mathbf{z}_a = f(\mathbf{x}_a)$, the \textit{positive} ($p$) is a semantically similar sample (e.g., from the same class): $\mathbf{z}_p = f(\mathbf{x}_p)$ and the \textit{negative} ($n$) is a semantically dissimilar sample (e.g., from a different class): $\mathbf{z}_n = f(\mathbf{x}_n)$.

The model is trained to pull the anchor and positive together while pushing the anchor and negative apart, leading to an embedding space in which semantic similarity is inversely related to Euclidean distance.
This triplet formulation provides a flexible framework for \textit{Semantic Structure Learning}, directly optimizing for similarity relationships rather than reconstruction error.

\subsection{Connection to Tensor Rank}
	While avoiding explicit rank constraints, our method maintains a connection to tensor algebra through the embedding dimension. 
	We now formalize the role of the diversity regularization term in determining the effective dimensionality of the learned representation.
\begin{remark}[Effective vs.\ Algebraic Rank]
	We emphasize that this notion of rank is \emph{effective} rather than algebraic: 
	it characterizes the number of active, linearly independent latent components induced by the embedding geometry and optimization dynamics, 
	rather than the minimal CP rank in the classical sense.
\end{remark}
\begin{proposition}[Embedding Dimension as Effective Rank]\label{prop:embedding_rank}
	Let $\mathbf{Z} = [\mathbf{z}_1^\top, \dots, \mathbf{z}_n^\top]^\top \in \mathbb{R}^{n\times d}$ be the learned embedding matrix, 
	where rows correspond to data samples and columns to embedding dimensions. 
	Let each \textbf{column} be $\ell_2$-normalized, i.e., $\|\mathbf{z}^{(k)}\|_2 = 1$ for all $k=1,\dots,d$. 
	Define the correlation matrix $\mathbf{C} = \mathbf{Z}^\top\mathbf{Z} \in \mathbb{R}^{d\times d}$ and the diversity loss
	\begin{equation}\label{eq:dl}
		\mathcal{L}_{\text{div}} = \frac{1}{d(d-1)} \sum_{i \neq j} |\mathbf{C}_{ij}|.
	\end{equation}
	
	If $\mathcal{L}_{\text{div}} \to 0$ during optimization, then $\mathbf{C} \to \mathbf{I}_d$ and $\mathbf{Z}$ has full column rank.
	Consequently, the embedding dimension $d$ serves as an effective rank determined by optimization.
\end{proposition}
\begin{proof}
	By construction, the correlation matrix $\mathbf{C} = \mathbf{Z}^\top \mathbf{Z}$ is symmetric and positive semidefinite.
	Since each column of $\mathbf{Z}$ is $\ell_2$-normalized, the diagonal entries satisfy
	\[
	\mathbf{C}_{ii} = \langle \mathbf{z}^{(i)}, \mathbf{z}^{(i)} \rangle = 1 \quad \text{for all } i.
	\]
	
	The diversity loss penalizes all off-diagonal correlations. Hence
	\[
	\mathcal{L}_{\text{div}} \to 0 \quad \Rightarrow \quad \mathbf{C}_{ij} \to 0 \text{ for all } i \neq j.
	\]
	Therefore,
	\[
	\mathbf{C} \to \mathbf{I}_d.
	\]
	Since $\mathbf{Z}^\top \mathbf{Z}$ becomes the identity, the columns of $\mathbf{Z}$ become asymptotically orthonormal, and $\mathbf{Z}$ has full column rank:
	\[
	\operatorname{rank}(\mathbf{Z}) = d.
	\]
\end{proof}
\begin{remark}
	As $\mathbf{C} \to \mathbf{I}_d$, the nonzero singular values of $\mathbf{Z}$ concentrate around one. 
	Consequently, the effective rank measures satisfy
	\[
	\text{erank}_{\epsilon}(\mathbf{Z}) \to d, \quad
	\text{erank}_{\tau}(\mathbf{Z}) \to d, \quad
	\text{erank}_{\text{ent}}(\mathbf{Z}) \to d,
	\]
	for any $\epsilon < 1$ and $\tau \leq 1$. 
	Thus, the optimization learns representations with maximal effective rank.
\end{remark}
	The triplet-based objective induces an implicit rank structure in the learned representations.
	By minimizing intra-class distances, the embeddings of samples within the same semantic class concentrate on low-dimensional subspaces, while maximizing inter-class distances enforces separation across distinct subspaces.
	Through diversity and uniformity regularization, the overall embedding adapts its effective dimensionality to the intrinsic complexity of the data, without requiring explicit rank selection.
	This property is particularly advantageous in applications where semantic organization is more relevant than exact reconstruction, such as medical imaging, retrieval, and high-dimensional pattern analysis.
	\subsection{Metric-Induced Tensor Factorization}The previous proposition established how diversity regularization controls the effective rank of the embedding matrix. 
		We now extend this analysis to the full tensor setting, showing that the induced similarity tensor $\mathcal{S}$ 
		admits a classical CP decomposition whose rank is implicitly determined by the metric learning objective. 
		This result formally connects our no-rank framework to traditional tensor algebra.
		\begin{proposition}[Metric Learning Induces CP Structure]
		\label{prop:metric_factorization}
		Let $\mathcal{X} \in \mathbb{R}^{I_1 \times \cdots \times I_N}$ be an $N$-th order tensor, and let embeddings be learned through the metric learning objective with triplet loss $\mathcal{L}_{\text{triplet}}$ and diversity regularization $\mathcal{L}_{\text{div}}$ (Eq.~\eqref{eq:dl}). 
		
		After training, the learned embedding function $f$ produces mode-$n$ embeddings:
		\[
		\mathbf{z}_{i_n}^{(n)} = f(\mathcal{X}_{: \cdots i_n \cdots :}) \in \mathbb{R}^d.
		\]
		
		Define the similarity tensor $\mathcal{S}$ induced by the metric learning objective:
		\[
		\mathcal{S}_{i_1,\dots,i_N} = \left\langle \mathbf{z}_{i_1}^{(1)}, \mathbf{z}_{i_2}^{(2)}, \dots, \mathbf{z}_{i_N}^{(N)} \right\rangle.
		\]
		Then:
		\begin{enumerate}[label=(\alph*)]
			\item $\mathcal{S}$ admits a CP decomposition with exactly $d$ components:
			\[
			\mathcal{S} = \sum_{r=1}^d \mathbf{z}_{: r}^{(1)} \circ \mathbf{z}_{: r}^{(2)} \circ \cdots \circ \mathbf{z}_{: r}^{(N)},
			\]
			where $\mathbf{z}_{: r}^{(n)} = [z^{(n)}_{1 r}, \dots, z^{(n)}_{I_n r}]^\top$.
			
			\item The factor matrices $\mathbf{Z}^{(n)} = [\mathbf{z}_{1}^{(n)}, \dots, \mathbf{z}_{I_n}^{(n)}]^\top \in \mathbb{R}^{I_n \times d}$ are determined by the metric learning optimization:
			\[
			\mathbf{Z}^{(n)} = \arg\min_{\mathbf{Z}} \left[ \mathcal{L}_{\text{triplet}}(\mathbf{Z}) + \lambda \mathcal{L}_{\text{div}}(\mathbf{Z}) \right],
			\]
			where $\lambda > 0$ balances the objectives.
			
			\item When $\mathcal{L}_{\text{div}} \to 0$, each $\mathbf{Z}^{(n)}$ has full column rank, making $d$ the effective CP rank of $\mathcal{S}$. This rank emerges implicitly from the interplay between triplet loss (shaping semantic geometry) and diversity regularization (controlling dimensionality).
		\end{enumerate}
	\end{proposition}
	\begin{proof}
		The CP structure follows algebraically from expanding the inner product:
		\[
		\langle \mathbf{z}_{i_1}^{(1)}, \dots, \mathbf{z}_{i_N}^{(N)} \rangle = \sum_{r=1}^d z^{(1)}_{i_1 r} \cdots z^{(N)}_{i_N r}.
		\]
		The key observation is that the factors $\mathbf{Z}^{(n)}$ are not free variables but outcomes of metric learning. Formally, let $\Theta$ parameterize $f$. The optimization:
		\[
		\min_{\Theta} \mathbb{E}_{(a,p,n)\sim\mathcal{T}} \left[ \|\mathbf{z}_a - \mathbf{z}_p\|^2 - \|\mathbf{z}_a - \mathbf{z}_n\|^2 + \alpha \right]_+ + \lambda \mathcal{L}_{\text{div}}(\mathbf{Z})
		\]
		determines $\mathbf{Z}^{(n)} = f(\mathcal{X}^{(n)}; \Theta^*)$, where $\mathcal{X}^{(n)}$ denotes mode-$n$ slices. Unlike classical CP where $\mathbf{Z}^{(n)}$ minimize $\|\mathcal{X} - \sum_r \mathbf{z}_{:r}^{(1)} \circ \cdots \circ \mathbf{z}_{:r}^{(N)}\|$, here they minimize semantic distance violations.\\
		When $\mathcal{L}_{\text{div}} \to 0$, we have $(\mathbf{Z}^{(n)})^\top\mathbf{Z}^{(n)} \to \mathbf{I}_d$, implying $\text{rank}(\mathbf{Z}^{(n)}) = d$. Thus the CP representation uses all $d$ components non-degenerately, making $d$ the effective rank.
\end{proof}
\begin{remark}[Interpretation of Metric-Induced Tensor Factorization]
	This result formalizes the connection between metric learning and tensor decomposition. While classical CP and Tucker methods decompose tensor entries by minimizing reconstruction error, our framework decomposes the similarity tensor $\mathcal{S}$ induced by embeddings optimized for semantic relationships. \\
	The key distinction is that factors $\mathbf{Z}^{(n)}$ are not obtained by solving $\min \|\mathcal{X} - \sum_r \mathbf{z}_{:r}^{(1)} \circ \cdots \circ \mathbf{z}_{:r}^{(N)}\|$, but rather emerge from metric learning: the triplet loss shapes embedding geometry so inner products encode semantic similarity, while diversity regularization controls dimensionality. Consequently, the CP structure is not imposed a priori but emerges from the optimization process, with effective rank determined implicitly rather than chosen explicitly.
	\end{remark}
\subsection{Triplet Loss with Regularization}
The fundamental triplet loss was popularized in FaceNet for deep metric learning \cite{schroff2015facenet}, and subsequent work has explored its optimization and sampling strategies \cite{wu2017sampling}. To regularize embedding spaces, recent research has highlighted the importance of alignment and uniformity properties for robust representations \cite{wang2020understanding}, and proposed variants that additionally encourage diversity or decorrelation among features \cite{mo2022rethinking}.

The core optimization objective is to ensure that the distance to negatives exceeds the distance to positives by at least a margin $\alpha$. This is enforced using the triplet loss combined with regularization terms to promote a well-structured embedding space.

The triplet loss is defined as:
\begin{equation}
	\mathcal{L}_{\text{triplet}} = \sum_{(a,p,n) \in \mathcal{T}} \left[ \|\mathbf{z}_a - \mathbf{z}_p\|_2^2 - \|\mathbf{z}_a - \mathbf{z}_n\|_2^2 + \alpha \right]_+,
\end{equation}
where $[x]_+ = \max(0, x)$ and $\alpha > 0$ is the margin parameter. This objective ensures that for all triplets:
\begin{equation}
	\|\mathbf{z}_a - \mathbf{z}_p\|_2^2 + \alpha < \|\mathbf{z}_a - \mathbf{z}_n\|_2^2.
\end{equation}

To prevent dimensional collapse and encourage the model to use all available dimensions efficiently, we employ the diversity penalty defined in \revisedd{Proposition~\ref{prop:embedding_rank}}
where $\mathbf{C} = \mathbf{Z}^\top \mathbf{Z}$ as defined above.

Following \cite{wang2020understanding}, we also promote a uniform distribution of embeddings on the unit sphere to avoid hubness and further improve generalization:
\begin{equation}
	\mathcal{L}_{\text{uniform}} = \log \mathbb{E}_{i,j \sim p_{\text{data}}} \left[ e^{-2\|\mathbf{z}_i - \mathbf{z}_j\|_2^2} \right].
\end{equation}

\subsection{Locality Preservation Framework}
	Unlike low-rank tensor decomposition methods (e.g., CP, Tucker, or t-SVD), which enforce global rank constraints by minimizing reconstruction error, the proposed framework prioritizes local semantic consistency. This shift enables flexible, non-linear embeddings that adapt to intrinsic data geometry rather than enforcing a fixed algebraic structure.
To ensure that local neighborhoods in the original high-dimensional space are preserved in the embedding space, we introduce locality preservation objectives. Let $\mathcal{N}_k(\mathbf{x}_i)$ denote the set of $k$-nearest neighbors of $\mathbf{x}_i$ in the original space. The objective is to encourage preservation of local neighborhoods, i.e.,
	\begin{equation}
	\mathbf{x}_j \in \mathcal{N}_k(\mathbf{x}_i) \;\Rightarrow\; f(\mathbf{x}_j) \text{ remains close to } f(\mathbf{x}_i).
\end{equation}

We quantify this using two standard metrics. \textit{Continuity} measures how well the original neighbors are represented in the embedding neighborhood:
\begin{equation}
	\text{Continuity} = 1 - \frac{1}{nk} \sum_{i=1}^n \sum_{j \in \mathcal{N}_k^E(\mathbf{x}_i)} \mathbb{I}\{j \notin \mathcal{N}_k^O(\mathbf{x}_i)\} \cdot r_O(i,j),
\end{equation}
where $\mathcal{N}_k^O$ and $\mathcal{N}_k^E$ are neighborhoods in the original and embedding spaces, respectively, and $r_O(i,j)$ is the rank of $j$ in $\mathbf{x}_i$'s original neighborhood. Here, continuity is computed using a rank-weighted penalty consistent with neighborhood overlap measures commonly used in manifold learning. \textit{Trustworthiness} measures the prevalence of "intruders" (points in the embedding neighborhood that were not in the original neighborhood):
\begin{equation}
	\text{Trustworthiness} = 1 - \frac{2}{nk(2n - 3k - 1)} \sum_{i=1}^n \sum_{j \in \mathcal{U}_k(\mathbf{x}_i)} (r_O(i,j) - k),
\end{equation}
where $\mathcal{U}_k(\mathbf{x}_i)$ is the set of intruders for point $i$. These quantities are used as evaluation metrics and are not directly optimized due to their non-differentiability.

To directly optimize for these properties, we augment the loss function with locality-preserving terms. The \textit{local consistency loss} ensures that original neighbors remain close in the embedding space:
\begin{equation}
	\mathcal{L}_{\text{local}} = \sum_{i=1}^n \sum_{j \in \mathcal{N}_k^O(\mathbf{x}_i)} \left[ \|f(\mathbf{x}_i) - f(\mathbf{x}_j)\|_2^2 - \delta_{\text{local}}\right]_+.
\end{equation}
Conversely, the \textit{global separation loss} ensures that non-neighbors remain far apart:
\begin{equation}
	\mathcal{L}_{\text{global}} = \sum_{i=1}^n \sum_{j \notin \mathcal{N}_k^O(\mathbf{x}_i)} \left[\delta_{\text{global}} - \|f(\mathbf{x}_i) - f(\mathbf{x}_j)\|_2^2\right]_+.
\end{equation}

The complete training objective is a weighted sum of all components:
\begin{equation}
	\mathcal{L}_{\text{total}} = \mathcal{L}_{\text{triplet}} + \lambda_1 \mathcal{L}_{\text{div}} + \lambda_2 \mathcal{L}_{\text{uniform}} + \lambda_3 \mathcal{L}_{\text{local}} + \lambda_4 \mathcal{L}_{\text{global}},
\end{equation}
where $\lambda_1, \lambda_2, \lambda_3, \lambda_4 > 0$ are hyperparameters that balance the loss terms.

\subsection{Network Architecture}
	The embedding function $f(\cdot)$ is realized as a deep neural network that maps inputs $\mathbf{x}$ onto the unit sphere. The encoder consists of $L$ fully connected layers with ReLU activations:
\begin{align}
	\mathbf{h}^{(1)} &= \text{ReLU}(\mathbf{W}^{(1)} \mathbf{x} + \mathbf{b}^{(1)}) \\
	\mathbf{h}^{(l)} &= \text{ReLU}(\mathbf{W}^{(l)} \mathbf{h}^{(l-1)} + \mathbf{b}^{(l)}) \quad \text{for } l = 2,\ldots,L \\
	\mathbf{z} &= \frac{\mathbf{W}^{(L+1)} \mathbf{h}^{(L)} + \mathbf{b}^{(L+1)}}{\|\mathbf{W}^{(L+1)} \mathbf{h}^{(L)} + \mathbf{b}^{(L+1)}\|_2}.
\end{align}
The final $\ell_2$-normalization ensures that embeddings lie on the unit sphere, facilitating uniformity regularization. Optionally, a non-linear projection head
	\begin{equation}
		\mathbf{p} = \mathbf{W}_2 \, \text{ReLU}(\mathbf{W}_1 \mathbf{z}),
	\end{equation}
	can be added to improve representation quality during training, similar to contrastive learning frameworks. By the universal approximation theorem \cite{hornik1989multilayer}, fully connected networks with sufficient width and depth can approximate any continuous mapping on compact domains, ensuring the encoder has enough capacity to capture complex non-linear relationships.
\subsection{Model Complexity}
	The proposed framework has a space complexity determined by the encoder network. For an $L$-layer fully connected network with layer widths $\{h_1,\dots,h_L\}$, input dimension $D$, and embedding dimension $d$, the total number of trainable parameters is $\mathcal{O}(D h_1 + \sum_{l=1}^{L-1} h_l h_{l+1} + h_L d)$. Unlike CP or Tucker decompositions, the method does not require storing tensor factors or specifying a rank. The time complexity for a single sample forward pass is of the same order, while training with triplets scales linearly with the number of triplets in a batch. Additional regularization terms contribute modest overhead: the diversity loss costs $\mathcal{O}(n d^2)$ per batch, and locality-preserving losses cost $\mathcal{O}(n k d)$ where we assume $k$-nearest neighbors are precomputed in the input space and fixed during training. In contrast, traditional tensor decompositions scale with both tensor size and rank (e.g., $\mathcal{O}(R \prod_m I_m)$ for CP), making the proposed approach more scalable for high-dimensional data while maintaining a flexible, rank-free embedding framework.\\
A schematic of the proposed no-rank metric learning framework for tensor data is dipicted in Figure~\ref{fig:method} (a detailed illustration can be found in Appendix~\ref{sec:app}, Figure~\ref{fig:modelfull}.)

	\begin{figure}[t]
		\centering
		\resizebox{\linewidth}{!}{%
			\begin{tikzpicture}[
				node distance=1.5cm and 2cm,
				every node/.style={font=\small},
				block/.style={draw, rounded corners, minimum width=2.8cm, minimum height=1.5cm, align=center, fill=white},
				arrow/.style={->, thick, >=stealth},
				dashedarrow/.style={->, thick, dashed, >=stealth},
				point/.style={circle, fill, inner sep=1.3pt}
				]

				\node[block, fill=gray!5] (tensor) at (0,3) {Tensor Data\\$\mathcal{X} \in \mathbb{R}^{I_1 \times \cdots \times I_N}$};
				\node[below=0.1cm of tensor] {\scriptsize Mode-n slices / fibers};

				\node[block, fill=gray!10] (triplet) at (5,3) {Triplet Sampling\\Anchor $(a)$\\Positive $(p)$\\Negative $(n)$};
				\draw[arrow] (tensor.east) -- (triplet.west);

				\node[block, fill=gray!15] (encoder) at (10,3) {Neural Encoder\\$f_\theta(\cdot)$};
				\draw[arrow] (triplet.east) -- (encoder.west);

				\node[block, fill=gray!20] (embed) at (15,3) {Embedding Space\\$\mathbf{z}_i \in \mathbb{R}^d$};
				\draw[arrow] (encoder.east) -- (embed.west);

				\node[block, fill=gray!10, minimum width=8cm] (loss) at (7.5,0) {Training Objective\\
					$\mathcal{L}_{\text{triplet}} + \mathcal{L}_{\text{div}} + \mathcal{L}_{\text{uniform}} + \mathcal{L}_{\text{local}} + \mathcal{L}_{\text{global}}$};
				
				\draw[arrow] (encoder.south) -- ++(0,-0.5) -| (loss.north);
				
				\draw[arrow] (triplet.south) -- ++(0,-0.5) -| (loss.north);
				
				\draw[dashedarrow] (loss.west) -- ++(-1.5,0) |- (encoder.south);
					\draw[arrow] (embed.south) -- ++(0,-0.5) -| (loss.north);

				\node[below left=0.3cm and 0cm of loss, align=right, font=\scriptsize] (loss_left) {
					Triplet loss:\\pull $a,p$ together,\\push $a,n$ apart
				};
				
				\node[below right=0.3cm and 0cm of loss, align=left, font=\scriptsize] (loss_right) {
					Regularization:\\diversity, uniformity,\\locality preservation
				};
				
				\draw[arrow, gray!50] (loss.south west) -- ++(0,-0.2) -| (loss_left.north);
				\draw[arrow, gray!50] (loss.south east) -- ++(0,-0.2) -| (loss_right.north);
				
			\end{tikzpicture}
		}
		\caption{Overview of the proposed no-rank metric learning framework. Tensor slices are embedded via a neural encoder and optimized with triplet-based metric learning, combined with geometric and diversity regularization, producing a semantically structured embedding space. The embedding space shows tight \textit{intra-class clustering} and \textit{large inter-class separation}. For a more detailed illustration of the method see Figure~\ref{fig:modelfull}.}
		\label{fig:method}
	\end{figure}

\subsection{Convergence and Geometric Guarantees}
To provide theoretical context, we leverage generalization bounds based on Rademacher complexity \cite{shalev2014understanding, mohri2018foundations}.
	Metric learning theory supports locality preservation under Lipschitz mappings \cite{weinberger2005distance}.
	Optimization guarantees for SGD under smooth objectives have been established in \cite{bottou2018optimization, lee2016gradient}.
	Manifold geometry and neighborhood structure preservation have been studied in classical manifold learning literature \cite{hein2007graph, tenenbaum2000global}. We now establish theoretical guarantees for the proposed framework, focusing on optimization convergence and the geometric properties of the learned embedding space.

If the mapping $f$ is $L$-Lipschitz, local neighborhoods are distorted by at most a factor proportional to $L$, providing a sufficient condition for approximate locality preservation.

For a metric learning model with Rademacher complexity $\mathfrak{R}_n$ \cite{shalev2014understanding,mohri2018foundations}, the generalization error $\mathcal{E}_{\text{gen}}$ is bounded with probability at least $1-\delta$ by:
\begin{equation}
	\mathcal{E}_{\text{gen}} \leq \mathcal{E}_{\text{emp}} + O\left(\frac{\mathfrak{R}_n}{\sqrt{n}} + \sqrt{\frac{\log(1/\delta)}{n}}\right).
\end{equation}
The gradients of the triplet loss induce attractive forces between positive pairs and repulsive forces between negative pairs \cite{schroff2015facenet}. For an active triplet (where the loss is positive), the gradients are:
\begin{align}
	\nabla_{\mathbf{z}_a} \mathcal{L}_{\text{triplet}} &= 2(\mathbf{z}_a - \mathbf{z}_n) - 2(\mathbf{z}_a - \mathbf{z}_p), \\
	\nabla_{\mathbf{z}_p} \mathcal{L}_{\text{triplet}} &= 2(\mathbf{z}_p - \mathbf{z}_a), \\
	\nabla_{\mathbf{z}_n} \mathcal{L}_{\text{triplet}} &= 2(\mathbf{z}_a - \mathbf{z}_n).
\end{align}

\begin{lemma}[Convergence to a Critical Point]
	\label{lemma:convergence}
	Let the total loss function $\mathcal{L}_{\text{total}}(\theta)$ be $L$-smooth and bounded below. 
	Assume the stochastic gradients $g_t$ satisfy the conditional unbiasedness property
	\[
	\mathbb{E}[g_t \mid \mathcal{F}_t] = \nabla_\theta \mathcal{L}_{\text{total}}(\theta_t),
	\]
	and have uniformly bounded conditional variance
	\[
	\mathbb{E}\big[\|g_t - \nabla_\theta \mathcal{L}_{\text{total}}(\theta_t)\|^2 \mid \mathcal{F}_t\big] \leq \sigma^2
	\]
	almost surely, where $\mathcal{F}_t$ denotes the history up to time $t$. 
	When minimized using stochastic gradient descent with a learning rate schedule $\{\eta_t\}$ satisfying the Robbins--Monro conditions 
	$\sum_{t=1}^\infty \eta_t = \infty$ and $\sum_{t=1}^\infty \eta_t^2 < \infty$ \cite{robbins1951stochastic}, 
	the parameter sequence $\{\theta_t\}$ satisfies
	\[
	\liminf_{t \to \infty} \ \mathbb{E}\big[\|\nabla_{\theta} \mathcal{L}_{\text{total}}(\theta_t)\|^2\big] = 0.
	\]
	Consequently, there exists a subsequence of iterates $\{t_k\}$ such that
	\(\mathbb{E}\big[\|\nabla_{\theta} \mathcal{L}_{\text{total}}(\theta_{t_k})\|^2\big]\to 0\).
	If additionally the iterates $\{\theta_t\}$ are almost surely bounded (or $\mathcal{L}_{\text{total}}$ is coercive), then there exists a (random) subsequence $\theta_{t_k}$ that converges a.s. to a limit point $\theta^*$, and any such limit point satisfies $\nabla_{\theta}\mathcal{L}_{\text{total}}(\theta^*)=0$.
\end{lemma}

\begin{proof}
	The proof follows standard stochastic approximation arguments. Because $\mathcal{L}_{\text{total}}$ is $L$-smooth, the one-step descent inequality (taking total expectation over the noise history) yields
	\[
	\mathbb{E}\big[\mathcal{L}_{\text{total}}(\theta_{t+1})\big] 
	\leq \mathbb{E}\big[\mathcal{L}_{\text{total}}(\theta_t)\big] 
	- \eta_t \mathbb{E}\big[\|\nabla \mathcal{L}_{\text{total}}(\theta_t)\|^2\big] 
	+ \frac{L}{2}\eta_t^2 \mathbb{E}\big[\|g_t\|^2\big].
	\]
	Using conditional unbiasedness and the conditional variance bound gives (via law of total expectation)
	\[
	\mathbb{E}\big[\|g_t\|^2\big] 
	= \mathbb{E}\big[\|\nabla \mathcal{L}_{\text{total}}(\theta_t)\|^2\big] 
	+ \mathbb{E}\big[\|g_t - \nabla \mathcal{L}_{\text{total}}(\theta_t)\|^2\big]
	\le \mathbb{E}\big[\|\nabla \mathcal{L}_{\text{total}}(\theta_t)\|^2\big] + \sigma^2.
	\]
	Substituting and summing from $t=1$ to $T$ yields
	\[
	\sum_{t=1}^T \eta_t \mathbb{E}\big[\|\nabla \mathcal{L}_{\text{total}}(\theta_t)\|^2\big] 
	\leq \mathbb{E}\big[\mathcal{L}_{\text{total}}(\theta_1)\big] - \mathcal{L}^* 
	+ \frac{L\sigma^2}{2}\sum_{t=1}^T \eta_t^2,
	\]
	where $\mathcal{L}^*$ is the infimum of $\mathcal{L}_{\text{total}}$. The Robbins--Monro conditions ensure the right-hand side stays bounded as $T\to\infty$. Since $\sum_t\eta_t=\infty$, the only possibility is
	\[
	\liminf_{t\to\infty}\mathbb{E}\big[\|\nabla \mathcal{L}_{\text{total}}(\theta_t)\|^2\big]=0,
	\]
	which implies the existence of a subsequence $t_k$ with \(\mathbb{E}[\|\nabla \mathcal{L}(\theta_{t_k})\|^2]\to 0\). The final statement about parameter subsequence convergence follows from the additional boundedness/coercivity assumption (so that limit points exist), after which continuity of $\nabla\mathcal{L}$ implies any limit point is stationary.
\end{proof}

\begin{lemma}[Semantic Structure of the Embedding]
	\label{lemma:structure}
	Assume the data is separable with a margin $\gamma > 0$, i.e., the optimal embedding satisfies $\|\mathbf{z}_a - \mathbf{z}_p\|_2^2 + \gamma \leq \|\mathbf{z}_a - \mathbf{z}_n\|_2^2$ for all valid triplets $(a, p, n)$. Further, assume the embedding function $f$ is $L$-Lipschitz and the data manifold $\mathcal{M}$ is compact. Then, the learned embedding space exhibits the following semantic structure:
	\begin{enumerate}
		\item[(i)] Intra-class clusters are tight: for any two points $\mathbf{x}_i, \mathbf{x}_j$ from the same class, $\|f(\mathbf{x}_i) - f(\mathbf{x}_j)\|_2 \leq L \cdot \text{diam}(\mathcal{M}_c)$, where $\mathcal{M}_c$ is the connected component of the data manifold containing points of that class.
		\item[(ii)] Inter-class clusters are separated: for any two points $\mathbf{x}_i, \mathbf{x}_j$ from different classes, $\|f(\mathbf{x}_i) - f(\mathbf{x}_j)\|_2$ is lower-bounded by $\sqrt{\gamma}$.
	\end{enumerate}
	This structure ensures that local neighborhoods are preserved, and the embedding is well-suited for similarity search.
\end{lemma}

\begin{proof}
	The proof leverages the margin condition and Lipschitz continuity.\\
(i) {Intra-class tightness:} Consider two points $\mathbf{x}_i$ and $\mathbf{x}_j$ from the same class. Since they belong to the same class and the data manifold $\mathcal{M}$ is compact, there exists a path connecting them within their class component $\mathcal{M}_c$. By the Lipschitz continuity of $f$, we have:
		\[
		\|f(\mathbf{x}_i) - f(\mathbf{x}_j)\|_2 \leq L \cdot d_{\mathcal{M}}(\mathbf{x}_i, \mathbf{x}_j) \leq L \cdot \text{diam}(\mathcal{M}_c),
		\]
		where $\text{diam}(\mathcal{M}_c)$ is the diameter of the connected component of class $c$ in the data manifold. This provides a uniform upper bound on intra-class distances.\\
(ii) {Inter-class separation:} Let $\mathbf{x}_i$ and $\mathbf{x}_j$ be from different classes. Using the margin condition with anchor $a = \mathbf{x}_i$, positive $p = \mathbf{x}_i$ (trivially from the same class), and negative $n = \mathbf{x}_j$, we get:
		\[
		\|f(\mathbf{x}_i) - f(\mathbf{x}_i)\|_2^2 + \gamma \leq \|f(\mathbf{x}_i) - f(\mathbf{x}_j)\|_2^2.
		\]
		Simplifying yields $\gamma \leq \|f(\mathbf{x}_i) - f(\mathbf{x}_j)\|_2^2$, and thus $\sqrt{\gamma} \leq \|f(\mathbf{x}_i) - f(\mathbf{x}_j)\|_2$.
	The combination of (i) and (ii) demonstrates that the embedding maps the semantic structure of the data into a geometric structure where points from the same class form tight clusters (bounded by manifold geometry) well-separated from clusters of other classes (by at least $\sqrt{\gamma}$). The Lipschitz property ensures this mapping preserves local connectivity.
\end{proof}
\begin{lemma}[Metric Preservation via Triplet Loss and Lipschitz Mapping]
		\label{lemma:metric}
		Assume the data lies on a compact manifold $\mathcal{M}$ and is separable with margin $\gamma > 0$, i.e., for all valid triplets $(a,p,n)$:
		\[
		\|\mathbf{z}_a - \mathbf{z}_p\|_2^2 + \gamma \leq \|\mathbf{z}_a - \mathbf{z}_n\|_2^2.
		\]
		Let the embedding function $f:\mathcal{M} \to \mathbb{R}^d$ be $L$-Lipschitz:
		\[
		\|f(\mathbf{x}) - f(\mathbf{y})\|_2 \le L \|\mathbf{x} - \mathbf{y}\|_2 \quad \forall \mathbf{x}, \mathbf{y} \in \mathcal{M}.
		\]
		
		Then, for any two points $\mathbf{x}_i, \mathbf{x}_j \in \mathcal{M}$:
		
		\begin{enumerate}[label=(\roman*)]
			\item {Intra-class preservation:} if $\mathbf{x}_i, \mathbf{x}_j$ belong to the same class component $\mathcal{M}_c$,
			\[
			\|f(\mathbf{x}_i) - f(\mathbf{x}_j)\|_2 \le L \, d_{\mathcal{M}}(\mathbf{x}_i, \mathbf{x}_j) \le L \, \text{diam}(\mathcal{M}_c),
			\]
			so local semantic distances are preserved up to a factor $L$.
			
			\item {Inter-class separation:} if $\mathbf{x}_i, \mathbf{x}_j$ belong to different classes,
			\[
			\|f(\mathbf{x}_i) - f(\mathbf{x}_j)\|_2 \ge \sqrt{\gamma},
			\]
			providing a lower bound on distances across classes.
		\end{enumerate}
		
		Consequently, the embedding $f$ defines an \emph{approximate isometry} on the manifold with bounded distortion:
		\[
		\frac{1}{L} \|f(\mathbf{x}_i) - f(\mathbf{x}_j)\|_2 \le d_{\mathcal{M}}(\mathbf{x}_i, \mathbf{x}_j) \le L \|f(\mathbf{x}_i) - f(\mathbf{x}_j)\|_2 \quad \text{(for intra-class pairs)}.
		\]
		
	\end{lemma}
	
	\begin{proof}
	{Intra-class:} For $\mathbf{x}_i, \mathbf{x}_j \in \mathcal{M}_c$, Lipschitz continuity gives
		\[
		\|f(\mathbf{x}_i) - f(\mathbf{x}_j)\|_2 \le L \|\mathbf{x}_i - \mathbf{x}_j\|_2 \le L \, d_{\mathcal{M}}(\mathbf{x}_i, \mathbf{x}_j) \le L \, \text{diam}(\mathcal{M}_c),
		\]
		preserving local distances up to factor $L$.  
		
	{Inter-class:} For $\mathbf{x}_i, \mathbf{x}_j$ from different classes, consider a triplet with anchor $a = \mathbf{x}_i$, positive $p$ in the same class, and negative $n = \mathbf{x}_j$. The triplet loss implies
		\[
		\|\mathbf{z}_a - \mathbf{z}_p\|_2^2 + \gamma \le \|\mathbf{z}_a - \mathbf{z}_n\|_2^2 \implies \|f(\mathbf{x}_i) - f(\mathbf{x}_j)\|_2 \ge \sqrt{\gamma}.
		\]
		
		Combining the two results shows that $f$ approximately preserves intra-class manifold distances while enforcing inter-class separation, yielding a bounded-distortion embedding suitable for similarity-based tasks.
	\end{proof}
\begin{theorem}
Under the assumptions of Lemmas \ref{lemma:convergence}, \ref{lemma:structure}, 	\revised{and \ref{lemma:metric},} the proposed metric learning framework:
	
	\begin{enumerate}
		\item Converges to a configuration that is a local minimum of the total objective $\mathcal{L}_{\text{total}}$.
		\item Yields a semantically structured embedding space where intra-class distances are minimized and inter-class distances are maximized beyond a margin $\gamma$.
	\item Preserves pairwise distances on the data manifold up to a bounded distortion factor, i.e., for any $\mathbf{x}_i, \mathbf{x}_j \in \mathcal{M}$:
			\begin{equation}
				\frac{1}{L} \|\mathbf{z}_i - \mathbf{z}_j\|_2 \le d_\mathcal{M}(\mathbf{x}_i, \mathbf{x}_j) \le L \|\mathbf{z}_i - \mathbf{z}_j\|_2,
			\end{equation}
			where $\mathbf{z}_i = f(\mathbf{x}_i)$ and $d_\mathcal{M}$ is the geodesic distance on the manifold.
	\end{enumerate}
	This result bridges optimization guarantees with geometric structure, ensuring that the learned representation is both stable and semantically meaningful.
\end{theorem}

\begin{proof}
	The theorem follows from combining the guarantees of the two lemmas. From Lemma \ref{lemma:convergence}, the optimization converges to a critical point $\theta^*$ where  $\liminf_{t\to\infty} \mathbb{E}[\|\nabla_{\theta} \mathcal{L}_{\text{total}}(\theta_t)\|^2] = 0$. Under the strict saddle point assumption\footnote{The strict saddle property requires that every saddle point has at least one direction of negative curvature. Under this condition, gradient-based methods with random initialization almost surely avoid saddle points \cite{lee2016gradient}.}, gradient-based methods converge almost surely to local minima rather than saddle points. Thus, $\theta^*$ is a local minimum of $\mathcal{L}_{\text{total}}$ with high probability.
	
	From Lemma \ref{lemma:structure}, when the embedding function $f_{\theta^*}$ (parameterized by the locally optimal $\theta^*$) operates on separable data with margin $\gamma$, the resulting embedding space exhibits the semantic structure described in Lemma \ref{lemma:structure}. Specifically (i) intra-class distances are bounded by $L \cdot \text{diam}(\mathcal{M}_c)$, (ii)  inter-class distances are lower-bounded by $\sqrt{\gamma}$. Finally, due to Lemma~\ref{lemma:metric}, because $f$ is $L$-Lipschitz, the embedding approximately preserves pairwise distances on the manifold:
		\[
		\frac{1}{L} d_\mathcal{M}(\mathbf{x}_i, \mathbf{x}_j) \le \|f(\mathbf{x}_i) - f(\mathbf{x}_j)\|_2 \le L \, d_\mathcal{M}(\mathbf{x}_i, \mathbf{x}_j),
		\]
		providing a bounded distortion guarantee.\\
	The combination of these results guarantees that SGD finds a locally optimal parameter configuration that produces a semantically structured embedding space suitable for similarity tasks.
\end{proof}

\begin{remark}
		The theoretical guarantees depend on key parameters:
		\begin{enumerate}[label=(\roman*)]
			\item The margin $\gamma$ (from Lemma \ref{lemma:structure}) directly controls inter-class separation,
			\item The Lipschitz constant $L$ (from Lemma \ref{lemma:structure} and the metric preservation lemma) controls how well local neighborhoods are preserved and bounds the distortion of pairwise distances on the data manifold,
			\item The learning rate conditions (from Lemma \ref{lemma:convergence}) ensure optimization convergence.
		\end{enumerate}
In practice, the regularization terms $\mathcal{L}_{\text{div}}$ and $\mathcal{L}_{\text{uniform}}$ promote well-spread embeddings (supporting large $\gamma$), while $\mathcal{L}_{\text{local}}$ and $\mathcal{L}_{\text{global}}$ enforce neighborhood preservation and approximate isometry (related to the Lipschitz property).
	\end{remark}
\begin{assumption}[Manifold Hypothesis]
	The data is assumed to lie on a smooth, low-dimensional Riemannian manifold $\mathcal{M} \subset \mathbb{R}^D$. 
	The embedding function $f: \mathcal{M} \to \mathbb{R}^d$ is approximately isometric, meaning that Euclidean distances in the embedding space preserve the intrinsic (geodesic) distances on the manifold:
	\begin{equation}
		d_{\mathcal{M}}(\mathbf{x}_i, \mathbf{x}_j) \;\approx\; \|f(\mathbf{x}_i) - f(\mathbf{x}_j)\|_2,
	\end{equation}
	for all $\mathbf{x}_i, \mathbf{x}_j \in \mathcal{M}$.
\end{assumption}
This framework is particularly well-suited for applications like medical imaging, where semantic similarity and clinical relevance are more critical than pixel-perfect reconstruction accuracy, opening new possibilities for tensor-based analysis of complex spatio-temporal data.

\section{Metrics and Methodology Analysis}\label{sec:metric}

\subsection{Evaluation Metrics}\label{subsec:em}
We assess the quality of the learned embeddings using a comprehensive suite of metrics that evaluate clustering quality, structural preservation, and alignment with ground-truth labels.

{Clustering Quality Metrics} evaluate the intrinsic structure of the embeddings without using label information:\\
\textit{Silhouette Score (Sil)} measures how similar samples are to their own cluster compared to other clusters. For a sample $i$, it is computed as:
	\begin{equation}
		s(i) = \frac{b(i) - a(i)}{\max\{a(i), b(i)\}},
	\end{equation}
	where $a(i)$ is the mean intra-cluster distance and $b(i)$ is the mean nearest-cluster distance. Scores range from $-1$ to $+1$, with higher values indicating better clustering.
	
{Davies-Bouldin Index}(DB) quantifies the trade-off between cluster compactness and separation. For $k$ clusters, it is defined as:
	\begin{equation}
		\text{DB} = \frac{1}{k} \sum_{i=1}^k \max_{j \neq i} \left( \frac{\sigma_i + \sigma_j}{d(c_i, c_j)} \right),
	\end{equation}
	where $\sigma_i$ is the average distance from points in cluster $i$ to its centroid $c_i$, and $d(c_i, c_j)$ is the distance between centroids. Lower values indicate better clustering.
	
 \textit{Calinski-Harabasz Index}(CH) is defined as the ratio of between-cluster dispersion to within-cluster dispersion:
	\begin{equation}
		\text{CH} = \frac{\text{Tr}(B_k)}{\text{Tr}(W_k)} \times \frac{N - k}{k - 1},
	\end{equation}
	where $\text{Tr}(B_k)$ and $\text{Tr}(W_k)$ are the traces of the between-cluster and within-cluster dispersion matrices, respectively. Higher values indicate tighter and better-separated clusters.

{External Validation Metrics} measure the agreement between the discovered clusters and the ground-truth labels:
\textit{Adjusted Rand Index (ARI)} measures the similarity between two clusterings, corrected for chance. It is defined as:
	\begin{equation}
		\text{ARI} = \frac{\sum_{ij} \binom{n_{ij}}{2} - \left[\sum_i \binom{a_i}{2} \sum_j \binom{b_j}{2}\right] / \binom{n}{2}}{\frac{1}{2} \left[\sum_i \binom{a_i}{2} + \sum_j \binom{b_j}{2}\right] - \left[\sum_i \binom{a_i}{2} \sum_j \binom{b_j}{2}\right] / \binom{n}{2}},
	\end{equation}
	where $n_{ij}$ is the contingency table between true and predicted clusters, $a_i$ and $b_j$ are the row and column sums. ARI ranges from $-1$ to $1$, with higher values indicating better agreement.
	
 \textit{Normalized Mutual Information (NMI)} measures the mutual information between true and predicted clusters, normalized by the entropy of each:
	\begin{equation}
		\text{NMI} = \frac{2 \cdot I(Y; \hat{Y})}{H(Y) + H(\hat{Y})},
	\end{equation}
	where $I(Y; \hat{Y})$ is the mutual information and $H(\cdot)$ is the entropy. NMI ranges from $0$ to $1$, with higher values indicating better cluster alignment.

{Metric Learning Specific Metrics} directly evaluate the effectiveness of the triplet loss objective:\\
\textit{Separation Ratio}(SR) quantifies the ratio of inter-class to intra-class distances:
	\begin{equation}
		\text{Separation Ratio} = \frac{\mathbb{E}[\|f(\mathbf{x}_i) - f(\mathbf{x}_j)\|_2 \mid y_i \neq y_j]}{\mathbb{E}[\|f(\mathbf{x}_i) - f(\mathbf{x}_j)\|_2 \mid y_i = y_j]}.
	\end{equation}
Since the Separation Ratio is scale-invariant, its absolute magnitude depends on the relative contraction of intra-class distances and expansion of inter-class distances rather than the embedding norm; therefore, comparisons are meaningful across methods trained under identical preprocessing and normalization.
	A higher ratio indicates that the embedding space successfully pulls same-class samples together and pushes different-class samples apart.

{Structural Preservation Metrics} evaluate how well local geometric structure is preserved during dimensionality reduction:\\
 \textit{Trustworthiness} (Trust.) measures the preservation of local structure by penalizing for false neighbors (points that are neighbors in the embedding but not in the original space):
	\begin{equation}
		T(k) = 1 - \frac{2}{nk(2n - 3k - 1)} \sum_{i=1}^n \sum_{j \in \mathcal{U}_k(i)} (r_O(i,j) - k),
	\end{equation}
	where $\mathcal{U}_k(i)$ are the intruders for point $i$ and $r_O(i,j)$ is the rank of $j$ in the original space.
	
 \textit{Continuity}(Cont.) is the complementary measure to trustworthiness, which penalizes for missing neighbors (points that are neighbors in the original space but not in the embedding):
	\begin{equation}
		C(k) = 1 - \frac{2}{nk(2n - 3k - 1)} \sum_{i=1}^n \sum_{j \in \mathcal{V}_k(i)} (r_E(i,j) - k),
	\end{equation}
	where $\mathcal{V}_k(i)$ are the missing neighbors for point $i$ and $r_E(i,j)$ is the rank of $j$ in the embedding space.

These metrics provide complementary perspectives: Silhouette, DB, and CH indices focus on intrinsic cluster quality; ARI and NMI validate clustering against ground truth; the Separation Ratio directly measures metric learning effectiveness; and Trustworthiness and Continuity evaluate the preservation of local geometric structure.

\subsection{Visualization of High-Dimensional Embeddings}
To gain qualitative insights into the structure of the learned embeddings, we project them into two dimensions using both linear and non-linear techniques. The quality of these visualizations is inherently dependent on the distance metric learned by the proposed model.

We employ three standard dimensionality reduction methods Principal Component Analysis (PCA) \cite{doi:10.1080/14786440109462720} as a linear baseline that projects data onto the directions of maximal variance. t-Distributed Stochastic Neighbor Embedding (t-SNE) \cite{maaten2008visualizing} as a non-linear technique that preserves local similarities by minimizing the Kullback-Leibler divergence \cite{kullback1951information} between probability distributions in the high- and low-dimensional spaces:
\begin{equation}
 \text{KL}(P \| Q) = \sum_{i \neq j} p_{ij} \log \frac{p_{ij}}{q_{ij}},    
\end{equation}
where
        \begin{align}
         p_{j|i} &= \frac{\exp(-\|\mathbf{x}_i - \mathbf{x}_j\|^2 / 2\sigma_i^2)}{\sum_{k \neq i} \exp(-\|\mathbf{x}_i - \mathbf{x}_k\|^2 / 2\sigma_i^2)}, \quad
            q_{ij} = \frac{(1 + \|\mathbf{y}_i - \mathbf{y}_j\|^2)^{-1}}{\sum_{k \neq l} (1 + \|\mathbf{y}_k - \mathbf{y}_l\|^2)^{-1}}.
        \end{align}
And Uniform Manifold Approximation and Projection (UMAP) \cite{mcinnes2018umap} as a manifold learning technique that assumes the data is uniformly distributed on a Riemannian manifold. It constructs a topological representation and optimizes a low-dimensional equivalent.

We project the raw, flattened input data $\mathbf{X}_{\text{raw}} \in \mathbb{R}^{N \times D}$ to 2D using PCA to establish a baseline: $\mathbf{X}_{\text{raw}} \rightarrow \mathbf{Y}_{\text{PCA-raw}} \in \mathbb{R}^{N \times 2}$.
And to represent metric learning embedding space, we project the learned embeddings $\mathbf{Z} \in \mathbb{R}^{N \times d}$ using all three techniques $ \mathbf{Z} \rightarrow   \mathbf{Y}_{\alpha} \in \mathbb{R}^{N \times 2}$, where $\alpha = {\text{PCA}}, {\text{t-SNE}}, {\text{UMAP}} $.
Well-separated, tight clusters in these projections indicate a semantically coherent embedding space. To quantitatively validate these qualitative observations, we complement the visualizations with histograms of intra-class and inter-class distances:
\begin{equation}
D_{\text{intra}} = \{ \|\mathbf{z}_i - \mathbf{z}_j\|_2 : y_i = y_j \}, \quad D_{\text{inter}} = \{ \|\mathbf{z}_i - \mathbf{z}_j\|_2 : y_i \neq y_j \}.
\end{equation}
A clear separation between the $D_{\text{intra}}$ and $D_{\text{inter}}$ distributions confirms the patterns observed in the 2D projections.

\subsection{Triplet Mining Strategy Analysis}
The strategy for selecting triplets from the training data is critical for efficient and stable model convergence. We analyze two common strategies \cites{schroff2015facenet,wang2019multi,hermans2017defense}.

The \textit{Semi-Hard Negative Mining} strategy selects negatives that are farther from the anchor than the positive, but within the margin $\alpha$. This provides a steady, moderate learning signal.
\begin{algorithm}[H]
\caption{Semi-Hard Negative Mining}
\begin{algorithmic}[1]
\REQUIRE Embeddings $\mathbf{Z}$, labels $\mathbf{y}$, margin $\alpha$
\FOR{each anchor $\mathbf{z}_a$ with label $y_a$}
    \STATE Find a random positive $\mathbf{z}_p$ where $y_p = y_a$
    \STATE Compute positive distance $d_p = \|\mathbf{z}_a - \mathbf{z}_p\|_2$
    \STATE Find the set of negatives $\mathcal{N} = \{\mathbf{z}_n \mid y_n \neq y_a, \ d_p < \|\mathbf{z}_a - \mathbf{z}_n\|_2 < d_p + \alpha \}$
    \STATE Select a random negative from $\mathcal{N}$ (if non-empty)
\ENDFOR
\end{algorithmic}
\end{algorithm}
The \textit{Hard Negative Mining} is a more aggressive strategy selects the most challenging triplets by choosing the most distant positive and the closest negative for each anchor. This can lead to faster learning but also risks instability if the hard negatives are outliers or mislabeled.
\begin{algorithm}[H]
\caption{Hard Negative Mining}
\begin{algorithmic}[1]
\REQUIRE Embeddings $\mathbf{Z}$, labels $\mathbf{y}$
\FOR{each anchor $\mathbf{z}_a$ with label $y_a$}
    \STATE Find hardest positive: $\mathbf{z}_p^* = \arg\max_{\mathbf{z}_p: y_p = y_a} \|\mathbf{z}_a - \mathbf{z}_p\|_2$
    \STATE Find hardest negative: $\mathbf{z}_n^* = \arg\min_{\mathbf{z}_n: y_n \neq y_a} \|\mathbf{z}_a - \mathbf{z}_n\|_2$
    \STATE Use triplet $(\mathbf{z}_a, \mathbf{z}_p^*, \mathbf{z}_n^*)$
\ENDFOR
\end{algorithmic}
\end{algorithm}
The comparative performance of these strategies is evaluated and presented in the following section.

\section{Results}\label{sec:res}
	To comprehensively evaluate the proposed no-rank metric learning framework, we compare its performance against representative methods from three major categories: (i) classical linear and manifold-based embeddings, (ii) tensor decomposition–based approaches with explicit rank selection, and (iii) deep learning–based representation and clustering models.\\
	\paragraph{(i) Linear and Manifold-Based Baselines:}
	We include widely used dimensionality reduction techniques followed by k-means clustering, namely PCA, t-SNE, and UMAP. These methods serve as standard baselines for assessing clustering quality and neighborhood preservation in low-dimensional embeddings, but do not explicitly model tensor structure or semantic similarity.\\
	\paragraph{(ii) Tensor Decomposition Methods:}
	To benchmark against classical tensor-based representations, we consider CP, Tucker, and t-SVD decompositions with varying ranks ($R = 5, 10, 20$). For each method, the input tensor is decomposed into low-rank factors, vectorized representations are extracted, and k-means clustering is applied. These approaches require explicit rank selection and primarily optimize reconstruction fidelity rather than discriminative or semantic objectives.\\
	\paragraph{(iii) Deep Learning Baselines:}
	We further compare against representative reconstruction-based deep clustering models, including Variational Autoencoders (VAE) \cite{kingma2013auto} and Deep Embedded Clustering (DEC) \cite{xie2016unsupervised}. These methods learn latent representations by optimizing reconstruction-driven objectives and perform clustering in the learned latent space, without explicitly enforcing metric or neighborhood constraints.\\
	\paragraph{Transformer-Based Models:}
	Transformer architectures are evaluated separately in Subsection~\ref{sec:comparison}, as they introduce fundamentally different inductive biases, optimization dynamics, and computational costs compared to both tensor decompositions and autoencoder-based models. A dedicated analysis allows for a more focused comparison with the proposed metric learning framework.\\
	In contrast to all baselines, the proposed method directly learns an embedding space optimized for semantic similarity using a triplet-based metric learning objective augmented with diversity, uniformity, and locality-preserving regularization. Unlike tensor decompositions, it does not impose explicit rank constraints, and unlike reconstruction-based deep models, it does not aim to reconstruct the original tensor. Instead, the embedding dimensionality acts as an implicit, optimization-driven notion of rank.\\
	All methods are evaluated using identical preprocessing, consistent train--test splits, and matched embedding dimensionality where applicable. Clustering is performed using k-means for all baselines to ensure comparability. Performance is assessed using clustering compactness (Silhouette, Davies--Bouldin), separation (Separation Ratio), neighborhood preservation (Continuity, Trustworthiness), and clustering agreement metrics (ARI, NMI). Dataset-specific results are reported in the following subsections.

\subsection{Metric Learning Performance in Face Recognition}

We evaluate this framework on face recognition, a canonical metric learning task. Here, the goal is to learn an embedding where the distance between an anchor image and a positive example (same person) is smaller than the distance to a negative example (different person): \(d(\text{anchor}, \text{positive}) < d(\text{anchor}, \text{negative})\).

We used two contrasting datasets (Table~\ref{tab:stat}): the {Labeled Faces in the Wild (LFW)} \cite{huang2008labeled} dataset, a medium-sized, imbalanced dataset representing a real-world challenge; and the {Olivetti Faces} \cite{pedregosa2011scikit} dataset, a smaller, balanced dataset captured in a controlled environment.

\begin{table}[h!]
	\centering
	\caption{Summary of Face Recognition Dataset Properties}
	\begin{tabular}{lcc}
		\toprule
		\textbf{Property} & \textbf{LFW Faces} & \textbf{Olivetti Faces} \\
		\midrule
		Total Images & 1,288 & 400 \\
		Identities & 7 & 40 \\
		Image Dimensions & $50 \times 37$ & $64 \times 64$ \\
		Samples per Person & 71--530 & 10 \\
		Class Distribution & Highly Imbalanced & Perfectly Balanced \\
		Environment & Unconstrained & Controlled \\
		\bottomrule
	\end{tabular}
	\label{tab:stat}
\end{table}

\subsubsection*{Quantitative Clustering Performance}

The clustering results (Table~\ref{tab:face_clustering_results}) demonstrate the decisive advantage of the metric learning approach over baselines like PCA and tensor decomposition methods for creating semantically meaningful clusters.

\begin{table}[htbp]
	\centering
	\caption{Clustering Performance on LFW and Olivetti Datasets}
		\resizebox{\linewidth}{!}{%
	\begin{tabular}{llccccccc}
		\toprule
		\textbf{Dataset} & \textbf{Method} & \textbf{Sil.} & \textbf{DB} & \textbf{SR} & \textbf{Cont.} & \textbf{Trust.} &\textbf{AIR}&\textbf{NMI}\\
		\midrule
		\multirow{11}{*}{\textbf{LFW}}
		& PCA + K-Means           & -0.0186 & 7.3302 & 1.0131 & \textbf{0.9967} & \textbf{0.9933} &0.0181&0.0324\\
		& t-SNE + K-Means         & -0.0922 & 212.4068 & 1.0005 & 0.9206 & 0.9487&0.0128&0.0347 \\
		& UMAP + K-Means          & -0.0815 & 41.6583 & 1.0013 & 0.9198 & 0.8721 &0.0049&0.0195\\
		& CP-R5                   & -0.1216 & 18.5653 & 1.0033 & 0.8186 & 0.8715&0.0066&0.0242 \\
		& CP-R10                  & -0.0461 & 13.2544 & 1.0234 & 0.8824 & 0.9072&0.0061&0.0272 \\
		& CP-R20                  & -0.0906 & 10.3420 & 0.9720 & 0.9090 & 0.9274&0.0061&0.3050 \\
		& Tucker-R5               & -0.0689 & 10.8919 & 1.0181 & 0.9504 & 0.9033 &0.0222&0.0479\\
		& Tucker-R10              & -0.0338 & 7.9771  & 1.0254 & 0.9837 & 0.9666& 0.0187&0.0450 \\
		& Tucker-R20              & -0.0037 & 5.8455  & 1.0280 & 0.9886 & 0.9789&0.0078&0.0296 \\
		& \revised{t-SVD - R5}            & \revised{-0.0690} &\revised{10.8911}   & \revised{1.0181} &\revised{0.9505}  & \revised{0.9335}&\revised{0.0233}& \revised{0.0477} \\
		& \revised{t-SVD - R10}            & \revised{-0.0338} &\revised{7.9748}   & \revised{1.0255} &\revised{0.9837}  & \revised{0.9667}&\revised{0.0196}&  \revised{0.0468}\\
		& \revised{t-SVD - R20}            & \revised{-0.0037} &\revised{5.8454}   & \revised{1.0280} &\revised{0.9886}  & \revised{0.9790}&\revised{0.0072}& \revised{0.0271} \\
		& \revised{VAE}           & \revised{0.0122} &\revised{5.2531}   & \revised{1.0508} &\revised{0.9879}  & \revised{0.9819}&\revised{0.0257} &\revised{0.0424}\\
		& \revised{DEC}           & \revised{-0.0468} &\revised{7.3699}   & \revised{1.0046} &\revised{0.8856}  & \revised{0.8436}&\revised{0.0167} &\revised{0.0273}\\
		& \textbf{Metric Learning}     & \textbf{0.9752} & \textbf{0.0566} & \textbf{49.1800} & 0.9236 & 0.9201 &\textbf{1.0000}&\textbf{1.0000}\\
		\midrule
		\multirow{11}{*}{\textbf{Olivetti}}
		& PCA + K-Means           & 0.1434  & 1.8243  & 1.6002 & \textbf{0.9982} & \textbf{0.9970}&0.3831&0.7275 \\
		& t-SNE + K-Means         & -0.0123 & 9.2275  & 2.3923 & 0.9449 & 0.9730&0.4737&0.7898 \\
		& UMAP + K-Means          & -0.1213 & 8.4268  & 2.1196 & 0.9399 & 0.9352&0.3072& 0.6700 \\
		& CP-R5                   & -0.1432 & 3.7012  & 1.6951 & 0.9163 & 0.8920&0.1115&0.5403 \\
		& CP-R10                  & -0.0514 & 2.9316  & 1.7121 & 0.9386 & 0.9384 &0.2250&0.6299\\
		& CP-R20                  & 0.0021  & 2.5596  & 1.5354 & 0.9624 & 0.9536&0.2397&0.6411 \\
		& Tucker-R5               & -0.0275 & 2.9707  & 1.7517 & 0.9664 & 0.9358&0.2374&0.6554 \\
		& Tucker-R10              & 0.1303  & 1.9508  & 1.6772 & 0.9891 & 0.9800&0.4185& 0.7542 \\
		& Tucker-R20              & 0.1845  & 1.6731  & 1.5764 & 0.9935 & 0.9864&0.41976&0.7956 \\
		& \revised{t-SVD - R5}  & \revised{-0.0275} &\revised{2.9707}   & \revised{1.7517} &\revised{0.9664}  & \revised{0.9358}&\revised{0.2374}& \revised{0.6550}\\
		& \revised{t-SVD - R10}            & \revised{0.1303} &\revised{1.9508}   & \revised{1.6772} &\revised{0.9891}  & \revised{0.9800}&\revised{0.4185}&\revised{0.7540} \\
		& \revised{t-SVD - R20}           & \revised{0.1845} &\revised{1.6731}   & \revised{1.5764} &\revised{0.9935}  & \revised{0.9864}&\revised{0.4976} &\revised{0.7950}\\
		& \revised{VAE}           & \revised{0.1724} &\revised{1.7179}   & \revised{1.6540} &\revised{0.9892}  & \revised{0.9855}&\revised{0.4859} &\revised{0.7906}\\
		& \revised{DEC}           & \revised{-0.0534} &\revised{3.1705}   & \revised{1.3689} &\revised{0.9163}  & \revised{0.8428}&\revised{0.1627} &\revised{0.5899}\\
		& \textbf{Metric Learning}     & \textbf{0.8566} & \textbf{0.2341} & \textbf{9.8471} & 0.9728 & 0.9827&\textbf{0.9580}&\textbf{0.9864} \\
		\bottomrule
	\end{tabular}}
	\label{tab:face_clustering_results}
\end{table}

On the challenging LFW dataset, proposed approach achieved a near-perfect Silhouette score of {0.9752}, a dramatic improvement over PCA ({-0.0186}). This is corroborated by the Davies-Bouldin index, which dropped from {7.33} to {0.0566}, and the Separation Ratio, which increased from {1.01} to {49.18}, indicating that inter-class distances became vastly larger than intra-class distances.

The same trend is clear on the Olivetti dataset, where metric learning outperformed all other methods, achieving a Silhouette score of {0.8566}, a Davies-Bouldin index of {0.2341}, and a Separation Ratio of {9.8471}.

\revisedd{The clustering performance of the proposed metric learning approach is a direct consequence of how the embedding is learned. 
	Unlike classical dimensionality reduction and tensor decomposition methods, which aim to preserve variance or minimize reconstruction error, the proposed metric learning model explicitly optimizes the embedding geometry for semantic separability through the metric learning objective. \\
	Specifically, the triplet loss enforces that samples of the same identity are embedded close together while pushing different identities far apart, producing a representation that is intrinsically cluster-friendly. 
	In contrast, methods such as PCA, CP, Tucker, and t-SVD focus on low-rank approximation of the input data, and nonlinear techniques such as t-SNE and UMAP emphasize neighborhood preservation rather than global class separation. Consequently, these baselines yield embeddings that may preserve structure but are not optimized for identity clustering.\\
	Furthermore, the proposed diversity regularization promotes decorrelation among embedding dimensions, resulting in a compact, well-conditioned representation with an effective rank matched to the number of discriminative factors. 
	This combination of discriminative training and structured regularization produces clusters that are more compact and better separated, which is reflected across all evaluation metrics in Table~\ref{tab:face_clustering_results}.
}

\subsubsection*{The Clustering-Structure Preservation Trade-off}

A key finding is the trade-off between cluster quality and local structure preservation. While PCA achieves near-perfect Continuity and Trustworthiness, it fails to form meaningful clusters for face identity. This is because PCA preserves the \textit{original pixel-level geometry}, which does not align with \textit{semantic identity}.

In contrast, metric learning deliberately distorts the original geometry to create a new, semantically-organized space. The lower Continuity and Trustworthiness scores are a direct consequence of this transformation: neighbors in the pixel space (e.g., similar lighting) are pulled apart if they depict different people, while images of the same person are brought together despite pixel-level differences. This trade-off is not a failure but the intended behavior, prioritizing task-relevant semantic separation over raw structural fidelity.

\subsubsection*{Qualitative and Visual Analysis}

Visualizations and retrieval examples qualitatively validate the quantitative results. Figure~\ref{fig:vis_comparison} shows t-SNE projections of the embeddings. Unlike the original data where classes overlap, the metric-learned embeddings form tight, well-separated clusters for each identity.

\begin{figure}
	\centering
	\includegraphics[width=0.24\linewidth]{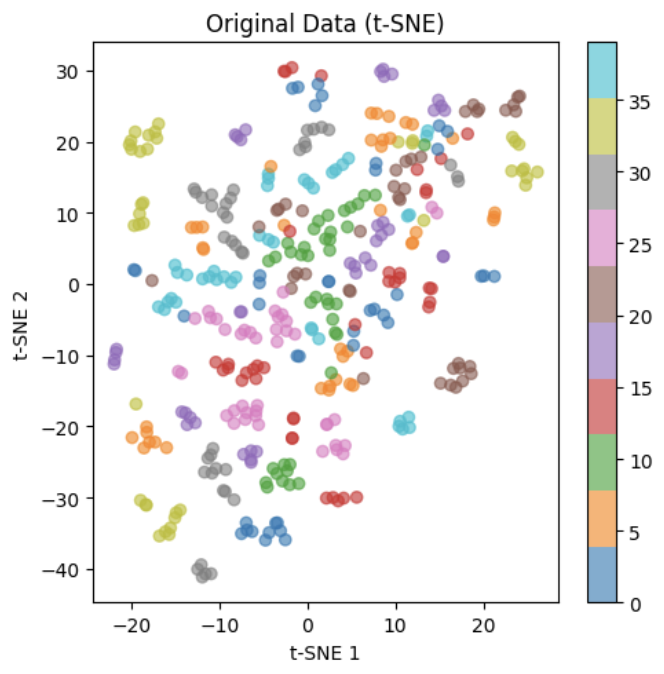}
	\includegraphics[width=0.235\linewidth]{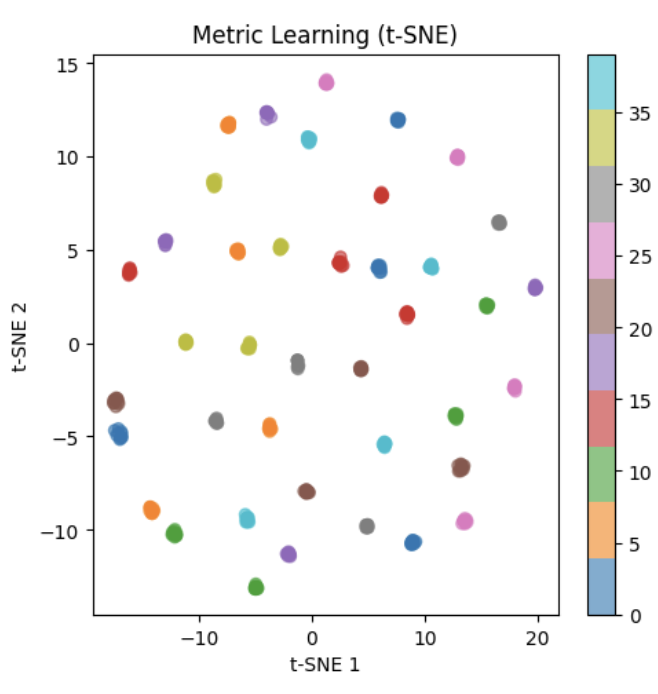}
	\includegraphics[width=0.245\linewidth]{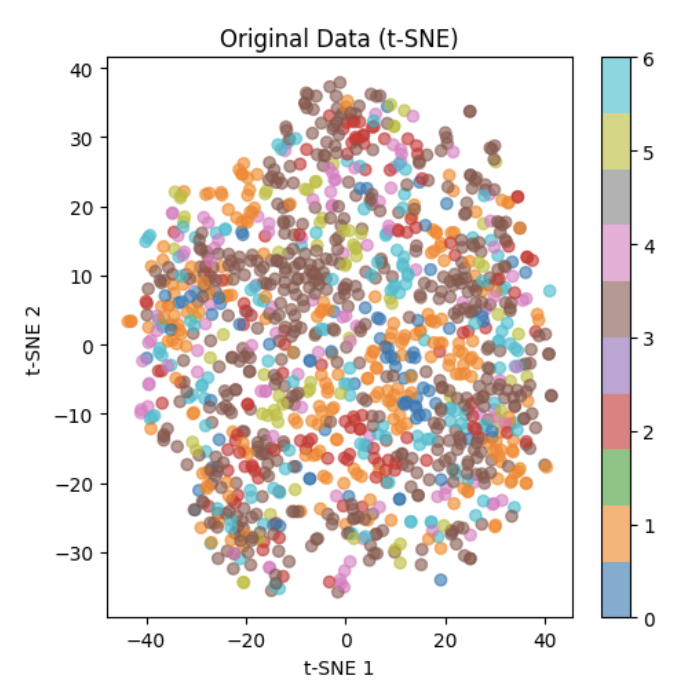}
	\includegraphics[width=0.235\linewidth]{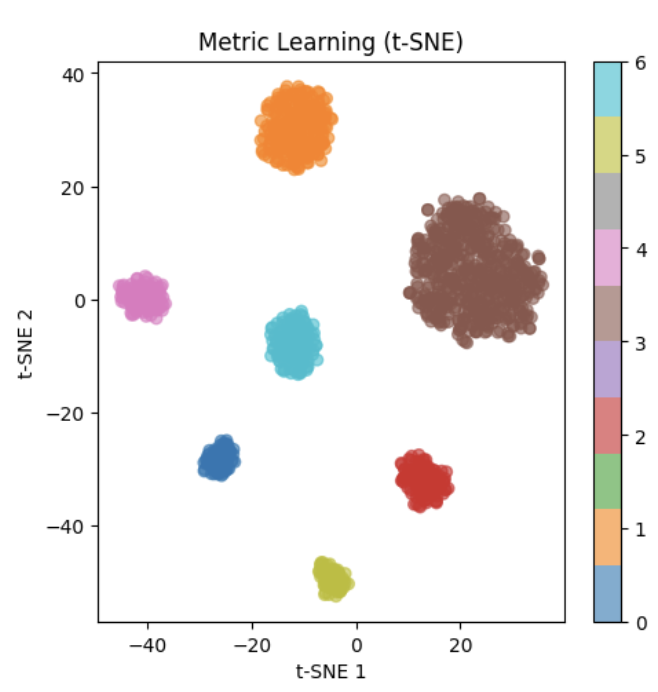}
	\caption{\revised{\textbf{LFW \& Olivetti Faces}}: t-SNE visualization of face embeddings. {Left to right:} Olivetti (original data), Olivetti (metric learning), LFW (original data), LFW (metric learning). Metric learning produces distinct, identity-based clusters.}
	\label{fig:vis_comparison}
\end{figure}

The practical impact of this clustering is evident in nearest-neighbor retrieval. In Figures~\ref{fig:clus} and \ref{fig:clus-o}, the nearest neighbors of a query image in the metric learning space are consistently of the same person, despite variations in pose and lighting. The small distances between anchor and positive pairs confirm the model successfully compacts same-identity samples.

\begin{figure}
	\centering
	\includegraphics[width=\linewidth]{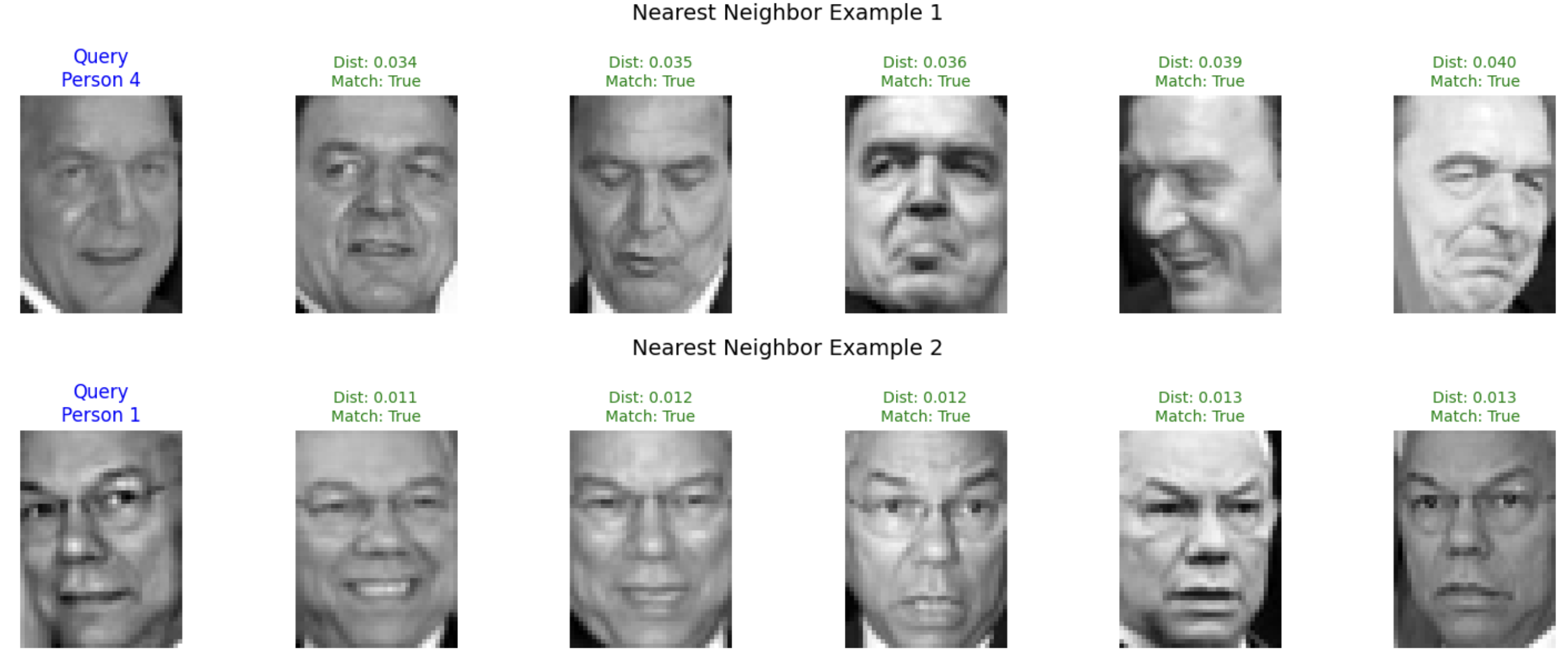}
	\caption{\revised{\textbf{LFW Faces}}: Nearest-neighbor retrieval on the LFW dataset using the metric learning embedding. The model correctly identifies same-identity images as the closest neighbors, with small corresponding Euclidean distances.}
	\label{fig:clus}
\end{figure}

\begin{figure}
	\centering
	\includegraphics[width=\linewidth]{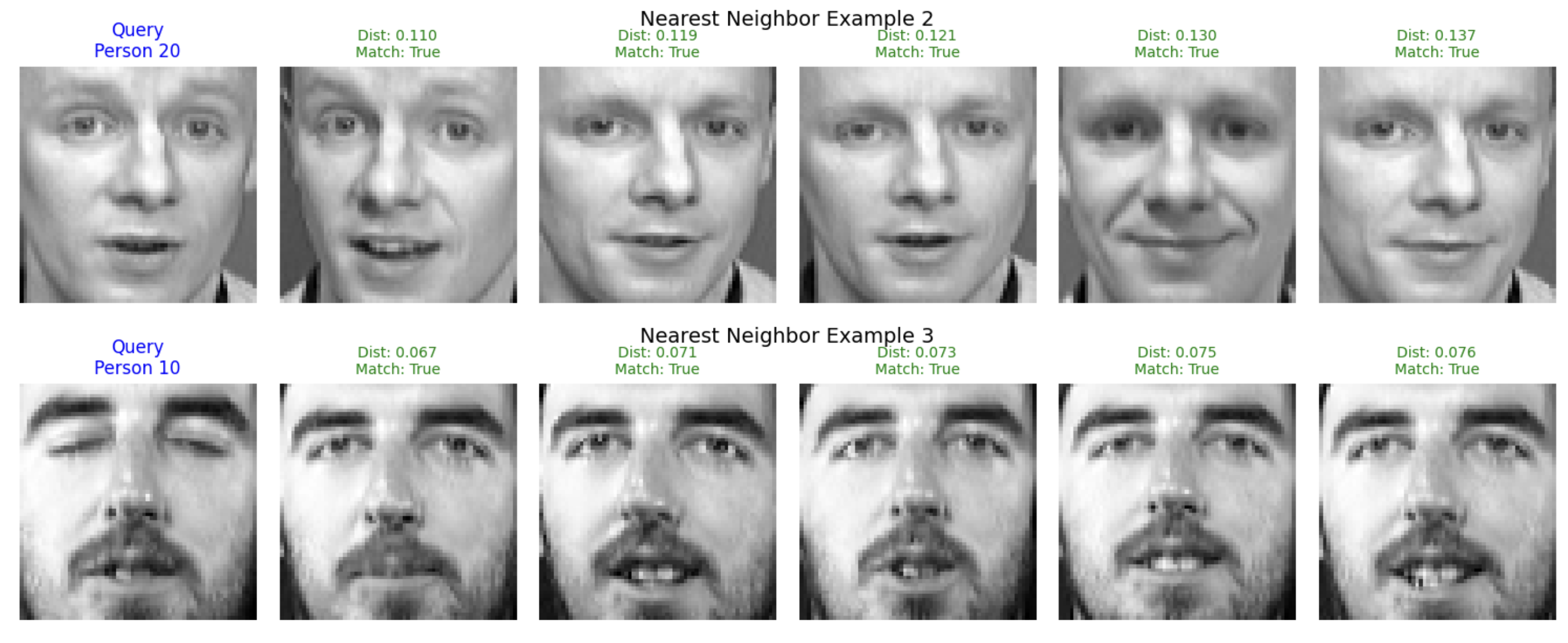}
	\caption{\revised{\textbf{Olivetti Faces}}: Nearest-neighbor retrieval on the Olivetti Faces dataset using the metric learning embedding. The results demonstrate robust, identity-based retrieval.}
	\label{fig:clus-o}
\end{figure}
\newpage 

The rank sensitivity analysis reveals a fundamental limitation of fixed-rank tensor decompositions for semantic clustering tasks. On the challenging LFW dataset,  CP, \revised{Tucker and t-SVD} decompositions fail to achieve meaningful clustering across all ranks, with Silhouette scores remaining near zero or negative. This indicates that the low-rank constraints destroy the semantic structure necessary for identity separation. While Tucker \revised{and t-SVD decompositions show} a slight improvement with higher ranks (from -0.0689 at R5 to -0.0037 at R20 \revised{for Tucker and from -0.0690 at R5 to -0.0037 at R20 for t-SVD}), the gains are minimal. The Olivetti dataset, with its controlled conditions, shows more rank sensitivity, particularly for Tucker \revised{and t-SVD decompositions} where performances improve substantially with higher ranks (\revised{Tucker}: 0.1303 at R10 to 0.1845 at R20 \revised{similat to t-SVD)}. However, even the best fixed-rank result (Tucker-R20\revised{/t-SVD}: 0.1845) is \revisedd{outperformed by the metric learning approach (0.8566), demonstrating that rank constraints inherently limit the ability to capture semantically meaningful representations, regardless of rank selection.}

\subsection{Metric Learning Performance on Brain Connectivity Data}
\revised{Unlike unsupervised baselines, the metric learning model leverages diagnostic labels during training to shape the embedding geometry.}
To evaluate metric learning framework on a complex, high-dimensional biomedical problem, we applied it to the Autism Brain Imaging Data Exchange (ABIDE) \cite{di2017enhancing} dataset. This dataset contains resting-state functional MRI (rs-fMRI) data from 871 subjects (403 with Autism Spectrum Disorder, ASD, and 468 typically developing controls). Each subject is represented by a $111 \times 111$ functional connectivity matrix, which are flattened and normalized to serve as input. The clinical labels (ASD vs. control) provide the semantic similarity relationships for guiding the metric learning process.

\begin{table}[H]
\centering
\caption{Clustering Performance on the ABIDE Brain Connectivity Dataset}
\resizebox{\linewidth}{!}{%
\begin{tabular}{lcccccccc}
\toprule
\textbf{Method} & \textbf{Sil.} & \textbf{DB} & \textbf{CH} & \textbf{SR} & \textbf{Cont.} & \textbf{Trust.}& \textbf{AIR}& \textbf{NMI} \\
\midrule
PCA + K-Means             & 0.2747 & 1.2721 & 53.5725   & 1.0017 & {0.9095} & {0.9078} & -0.0012   &0.0012\\
t-SNE + K-Means           & 0.5602 & 0.6062 & 228.7278  & 0.9919 & 0.9007 & 0.9010&-0.0011    &0.0000 \\
UMAP + K-Means            & 0.6194 & 0.5205 & 311.3453  & 0.9960 & 0.9007 & 0.8758& -0.0008   &0.0002 \\
CP-R5 &0.0022&17.8678&2.3617&1.0024&0.9144&0.8885&-0.0012&0.0006\\
CP-R10 &0.0025&12.2807&4.0502&1.0007&0.9100&0.9007&0.0005&0.0005\\
CP-R20 &0.0070&7.1090&13.8142&1.0035&0.8519&0.8600&0.0070&0.0170\\
Tucker-R5 &0.0061&9.1750&8.8522&1.0050&0.8969&0.8698&0.0016&0.0012\\
Tucker-R10 &0.0046&11.1730&6.4451&1.0040&0.9245&0.8825&0.0047&0.0033\\
Tucker-R20 &0.0053&11.7691&5.9515&1.0048&\textbf{0.9456}&0.8597&0.0014&0.0012\\
\revised{t-SVD- R5}              & \revised{0.0061} &\revised{9.1750} & \revised{8.8522}   & \revised{1.0050} & \revised{0.8969} & \revised{0.8698} & \revised{0.0016}& \revised{0.0012}\\
\revised{t-SVD- R10}              & \revised{0.0046} &\revised{11.1730} & \revised{6.4451}   & \revised{1.0040} & \revised{0.9245} & \revised{0.8825} & \revised{0.0047}& \revised{0.0033}\\
\revised{t-SVD- R20}              & \revised{0.0053} &\revised{11.7691} & \revised{5.9515}   & \revised{1.0048} & \revised{\textbf{0.9456}} & \revised{0.8597} & \revised{0.0014}& \revised{0.0012}\\
 \revised{VAE}           & \revised{0.0213} &\revised{5.9506}  & \revised{8.9761}& \revised{1.0226} &\revised{0.4978}  & \revised{0.5587}&\revised{0.0145} &\revised{0.0129}\\
 \revised{DEC}           & \revised{0.1166} &\revised{2.4894}& \revised{12.0923}  & \revised{1.1406} &\revised{0.5858}  & \revised{0.5890}&\revised{0.5072} &\revised{0.4091}\\
\textbf{Metric Learning}  & \textbf{0.9932} & \textbf{0.0186} & \textbf{31912.9395} & \textbf{1.9997} & {0.8389} &\textbf{0.9155}& \textbf{0.3002}   & \textbf{0.2372} \\
\bottomrule
\end{tabular}%
}
\label{tab:brain_clustering_results}
\end{table}

We have implemented a comprehensive data augmentation and normalization pipeline for brain connectivity matrices to address the dual challenges of limited neuroimaging datasets and inter-subject variability. The approach begins with patient-wise normalization to standardize individual connectivity profiles while preserving relative network topology. For each patient's correlation matrix $\mathbf{C}_i \in \mathbb{R}^{N \times N}$, we apply Z-score normalization: $\mathbf{C}_i' = \frac{\mathbf{C}_i - \mu_i}{\sigma_i}$, where $\mu_i$ and $\sigma_i$ are the mean and standard deviation computed exclusively from the $i$-th patient's connectivity matrix. This ensures each subject's data is centered and scaled independently while maintaining the intrinsic structure of their functional brain networks.

Following normalization, we employ multi-strategy data augmentation to enhance model robustness. For each normalized connectivity matrix $\mathbf{C}_i'$, we generate augmented variants through (i) Gaussian Noise Injection: $\mathbf{C}_i'' = \mathbf{C}_i' + \epsilon$ where $\epsilon \sim \mathcal{N}(0, \sigma^2\mathbf{I})$ with reduced noise variance $\sigma=0.02$ to accommodate the normalized data distribution, (ii) Symmetric Structure Preservation: $\mathbf{C}_i''' = \frac{1}{2}(\mathbf{C}_i'' + {\mathbf{C}_i''}^\top)$ to maintain mathematical consistency as a symmetric correlation matrix, (iii) Diagonal Identity Enforcement: $\mathbf{C}i'''{kk} = 1$ for $k=1,\ldots,N$ to preserve self-connectivity representation and (iv) Value Range Clipping: $\mathbf{C}_i^{\text{final}} = \text{clip}(\mathbf{C}_i''', -3, 3)$ using wider bounds appropriate for normalized data distributions.

During training, we further apply on-the-fly augmentation with probability $p=0.5$, introducing gentle noise perturbations ($\sigma=0.02$) to prevent overfitting while maintaining the normalized data characteristics.

This combined normalization-augmentation strategy effectively addresses both inter-subject variability through patient-wise standardization and dataset limitations through structural-preserving augmentation, enabling robust metric learning while respecting the neurobiological integrity of functional connectivity patterns. The normalization ensures comparability across subjects, while the augmentation introduces controlled variations that enhance model generalization without distorting the fundamental network topology.

\revised{The clustering results in Table~\ref{tab:brain_clustering_results} demonstrate that the proposed metric learning framework yields substantially more discriminative embeddings than all baseline methods. Owing to its supervised objective, the learned embedding space exhibits extremely tight intra-class compactness and strong inter-class separation, reflected by a very high Silhouette score (0.9932) and a low Davies–Bouldin index (0.0186). The large Calinski–Harabasz value further indicates strong between-group separation relative to within-group dispersion, which is expected when semantic labels explicitly guide the geometry of the embedding space.
	
	In contrast, unsupervised baselines—including UMAP, t-SNE, and all fixed-rank tensor decompositions—optimize geometric or reconstruction-based objectives without access to diagnostic labels. As a result, while some methods (e.g., UMAP) yield visually compact clusters, they fail to align with clinical structure, as evidenced by near-zero ARI and NMI scores. This discrepancy highlights a fundamental limitation of reconstruction-driven and rank-constrained tensor methods in clinical settings, where discriminative structure does not necessarily correspond to low-rank variance.
	
	Notably, CP, Tucker, and t-SVD decompositions exhibit strong sensitivity to rank selection and consistently produce negligible external validation scores across all tested ranks (Silhouette < 0.01, ARI/NMI < 0.02). By contrast, the proposed no-rank metric learning approach avoids explicit rank tuning and directly optimizes clinically meaningful similarity, resulting in substantially improved alignment with diagnostic labels (ARI = 0.3002, NMI = 0.2372). These results suggest that discriminative metric learning provides a more appropriate inductive bias for brain connectivity analysis than reconstruction-based tensor factorization.}

In summary, by directly optimizing for the clinically relevant separation between ASD and control subjects, metric learning uncovers a more discriminative and potentially more informative structure in brain connectivity data than methods focused solely on data reconstruction or geometric preservation.

\subsection{Metric Learning Performance on Simulated Datasets}

We evaluate the metric learning framework on two simulated datasets, each presenting distinct visual classification challenges relevant to their respective domains.

\textit{Galaxy Morphology Classification:} This dataset contains 500 ($64\times64$) pixel images simulating four galaxy morphological classes: Elliptical (smooth distribution), Spiral (prominent arms), Lenticular (disk without arms), and Irregular (asymmetric clumps). This task addresses fundamental challenges in astronomical image analysis.

\textit{Crystal Structure Prediction:} This dataset consists of 400 ($64\times64$) pixel images representing four crystal systems: Cubic (square symmetry), Hexagonal (six-fold symmetry), Tetragonal (rectangular symmetry), and Orthorhombic (anisotropic spacing). These patterns correspond to fundamental lattice structures in materials science.

Figure~\ref{fig:pat} shows representative samples from both datasets, illustrating the distinct visual characteristics of each class.

\begin{figure}
	\centering
	\includegraphics[width=\linewidth]{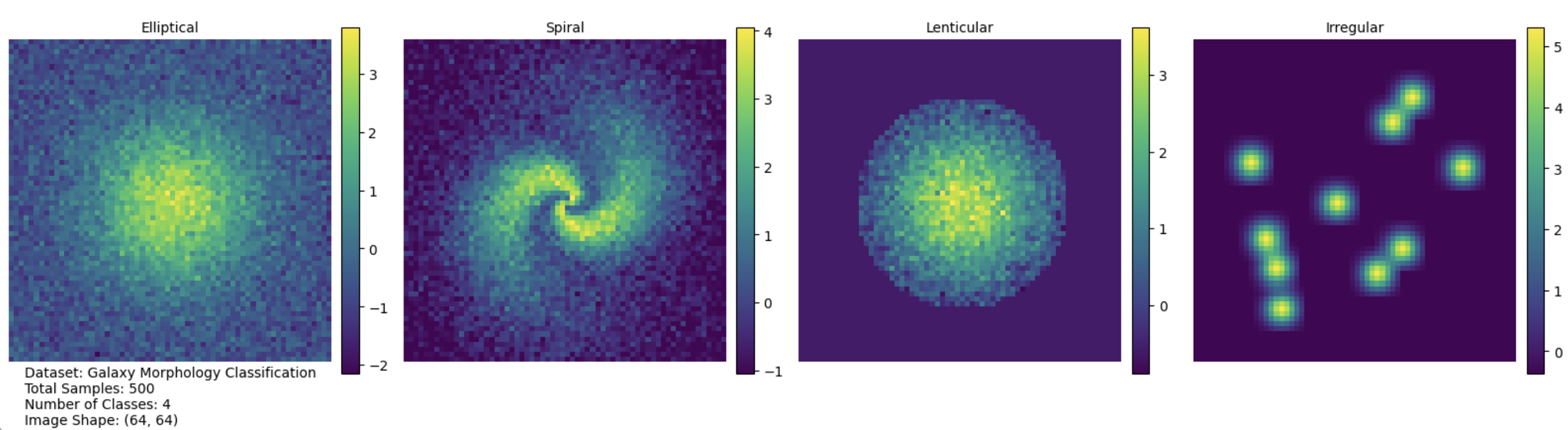}
	\includegraphics[width=\linewidth]{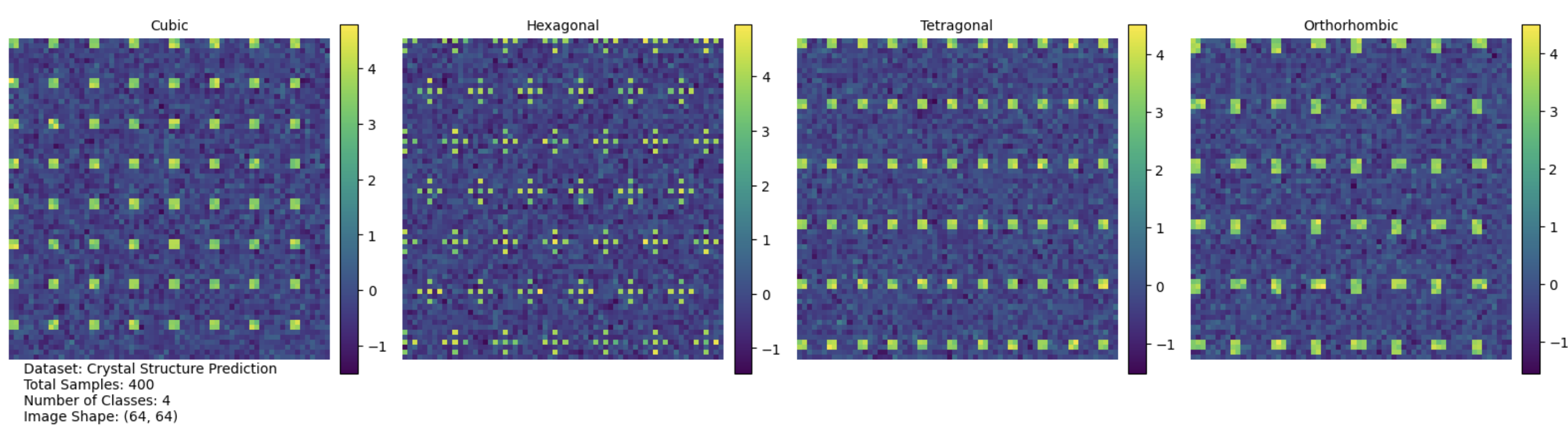}
	\caption{\revised{\textbf{Galaxy morphology \& Crystal structure:}}Example images from simulated datasets. {Top:} Galaxy morphology classes (Elliptical, Spiral, Lenticular, Irregular). {Bottom:} Crystal structure classes (Cubic, Hexagonal, Tetragonal, Orthorhombic).}
	\label{fig:pat}
\end{figure}

\subsubsection*{Quantitative Clustering Performance}

Table~\ref{tab:gal} presents clustering performance across multiple methods. The metric learning approach consistently achieves superior results on both datasets, substantially outperforming unsupervised dimensionality reduction techniques.

\begin{table}
	\centering
	\caption{Clustering Performance on Galaxy Morphology and Crystal Structure Datasets}
	\resizebox{\linewidth}{!}{%
	\begin{tabular}{llcccccccc}
		\toprule
		\textbf{Data} & \textbf{Method} & \textbf{Sil.} & \textbf{DB} & \textbf{CH} & \textbf{SR} & \textbf{Cont.} & \textbf{Trust.} & \textbf{AIR}& \textbf{NMI}\\
		\midrule
		\multirow{11}{*}{\textbf{Gal.}}
		& {\small PCA + K-Means  }          & 0.6572 & 2.8360 & $1.08 \times 10^{2}$   & 2.6090 & 0.8394 & 0.8073& 0.6245&0.7916  \\
		&{\small  t-SNE + K-Means  }        & 0.4382 & 0.9575 & $1.66 \times 10^{3}$   & 3.3263 & \textbf{0.8908} & 0.7830&0.5304 &0.6154  \\
		& {\small UMAP + K-Means}           & 0.4234 & 0.9435 & $8.97 \times 10^{2}$   & 2.4367 & 0.8502 & 0.7259 &0.5755 &0.6555\\
		& CP-R5                    & 0.6850 & 0.9077 & $1.93 \times 10^{3}$   & 5.3548 & 0.8004 & 0.7040 &0.5844 &0.7292\\
		& CP-R10                   & 0.6791 & 0.9805 & $6.72 \times 10^{2}$   & 3.4933 & 0.7961 & 0.7301&0.4916 & 0.6500\\
		& CP-R20                   & 0.5910 & 1.2845 & $2.00 \times 10^{2}$   & 2.4022 & 0.7671 & 0.7475&0.4637 &0.6237 \\
		& Tucker-R5                & 0.5051 & 1.0396 & $2.53 \times 10^{2}$   & 2.4988 & 0.8003 & 0.7216& 0.5577	&0.7042 \\
		& Tucker-R10               & 0.4940 & 2.0683 & $6.21 \times 10^{1}$   & 1.8776 & 0.7956 & 0.7674& 0.5699&0.7229 \\
		& Tucker-R20               & 0.2357 & 2.2079 & $1.85 \times 10^{1}$   & 1.6687 & 0.7704 & 0.8168 & 0.4036&0.5620\\
		&\revised{t-SVD- R5}              & \revised{0.7007} &\revised{1.8470} & \revised{3.01$\times 10^{2}$}   & \revised{3.2555} & \revised{0.8273} & \revised{0.7704} & \revised{0.4989}&\revised{0.6752}\\
		&\revised{t-SVD- R10}              & \revised{0.6995} &\revised{1.9915} & \revised{3.88$\times 10^{2}$}   & \revised{3.2914} & \revised{0.8258} & \revised{0.7681} & \revised{0.5134}&\revised{0.7264}\\
		&\revised{t-SVD- R20}              & \revised{0.3134} &\revised{2.0187} & \revised{1.92$\times 10^{1}$}   & \revised{2.5133} & \revised{0.7593} & \revised{0.7992} & \revised{0.4419}&\revised{0.6182}\\
		&\revised{VAE}           & \revised{0.5816} &\revised{1.7615}  &\revised{ 1.34$ \times 10^{3}$} & \revised{2.0874} &\revised{0.7866}  & \revised{0.7853}&\revised{0.5949} &\revised{0.7443}\\
		&\revised{DEC}           & \revised{0.4490} &\revised{1.4421}  & \revised{ 6.68$ \times 10^{3}$} & \revised{1.9701} &\revised{0.6951}  & \revised{0.7208}&\revised{0.5384} &\revised{0.7106}\\
		& \textbf{Metric Learning} & \textbf{0.9999} & \textbf{0.0001} & $\mathbf{4.01 \times 10^{9}}$ & \textbf{9275.82} & 0.8031 & \textbf{0.8208}&\textbf{1.0000} &\textbf{0.9999} \\
		\midrule
		\multirow{11}{*}{\textbf{Cry.}}
		& {\small PCA + K-Means }           & 0.8843 & 0.1708 & $7.43 \times 10^{3}$   & 9.0601 & 0.9350 & 0.9315&0.9711 &0.9871 \\
		&{\small  t-SNE + K-Means    }      & 0.9193& 0.1031 & $2.15 \times 10^{4}$   & 16.0362 & \textbf{0.9521} & \textbf{0.9303} & 0.9832&0.9898\\
		& {\small UMAP + K-Means      }     & 0.9531 & 0.0651 & $6.86 \times 10^{4}$   & 28.3698 & 0.9333 & 0.9138 &0.9264 &0.9289\\
		& CP-R5                    & 0.9751 & 0.0355 & $3.55 \times 10^{5}$   & 61.8594 & 0.8863 & 0.8877 &0.9007 &0.9100\\
		& CP-R10                   & 0.9820 & 0.0254 & $3.72 \times 10^{5}$   & 62.9187 & 0.8904 & 0.8915&0.9015 & 0.9091\\
		& CP-R20                   & 0.9572 & 0.0597 & $7.48 \times 10^{4}$   & 27.7025 & 0.8912 & 0.8862 &0.9018 &0.9112\\
		& Tucker-R5                & 0.6455 & 0.6524 & $3.96 \times 10^{2}$   & 2.8324 & 0.8877 & 0.8906&0.8999 &0.9072 \\
		& Tucker-R10               & 0.1714 & 2.0658 & $4.98 \times 10^{1}$   & 1.3687 & 0.8488 & 0.9107 &0.9094 &0.9097\\
		& Tucker-R20               & 0.0632 & 3.1399 & $2.47 \times 10^{1}$   & 1.1532 & 0.8254 & 0.8440&0.9165 &0.9187 \\
		&\revised{t-SVD- R5}              & \revised{0.9098} &\revised{0.2767} & \revised{1.55$\times 10^4$}   & \revised{13.090} & \revised{0.8984} & \revised{0.8921} & \revised{0.8973}&\revised{0.9068}\\
		&\revised{t-SVD- R10}              & \revised{0.9256} &\revised{0.1091} & \revised{1.75$\times 10^4$}   & \revised{14.1705} & \revised{0.9183} & \revised{0.9109} & \revised{0.9291}&\revised{0.9085}\\
		&\revised{t-SVD- R20}              & \revised{0.9003} &\revised{0.2167} & \revised{9.92$\times 10^3$}   & \revised{8.8781} & \revised{0.9067} & \revised{0.9001} & \revised{0.9165}&\revised{0.9184}\\
		&\revised{VAE}           & \revised{0.9876} &\revised{0.0175}  &\revised{ 2.39$ \times 10^{5}$}& \revised{99.1187} &\revised{0.8761}  & \revised{0.8891}&\revised{\textbf{1.0000}} &\revised{\textbf{1.0000}}\\
		&\revised{DEC}           & \revised{0.5572} &\revised{0.6980}  & \revised{ 1.77 $\times 10^{4}$}& \revised{3.3112} &\revised{0.8864}  & \revised{0.8890}&\revised{\textbf{1.0000}} &\revised{\textbf{1.0000}}\\
		& \textbf{Metric Learning} & \textbf{1.0000} & \textbf{0.0001} & $\mathbf{1.19 \times 10^{10}}$ & \textbf{14095.10} & 0.8829 & 0.8862& \textbf{1.0000}&\textbf{1.0000}\\
		\bottomrule
	\end{tabular}
}
	\label{tab:gal}
\end{table}

\revised{Metric learning achieves near-perfect Silhouette scores ($\approx$1.0) and near-zero Davies–Bouldin indices on both simulated datasets, reflecting almost complete separation between classes. This behavior is expected in a supervised metric learning setting when class boundaries are well-defined by construction. The extremely large Calinski–Harabasz values and Separation Ratios arise from the explicit maximization of inter-class margins and contraction of intra-class distances, rather than from variance-based or reconstruction-driven objectives optimized by unsupervised baselines.}

This performance comes with the expected trade-off: metric learning exhibits moderately lower Continuity and Trustworthiness scores compared to some unsupervised methods. \revised{This indicates that while global cluster structure is optimized for class separation, some local neighborhood relationships from the original pixel space are intentionally distorted—a known and deliberate consequence of prioritizing semantic class identity over raw pixel similarity.}

\subsubsection*{Qualitative Visualization}

\revised{The quantitative superiority is visually confirmed in Figures~\ref{fig:clus-gal-sry1} and \ref{fig:clus-gal-sry} (see Appendix~\ref{sec:app}).  The metric-learned embeddings form tight, well-separated clusters with minimal overlap, contrasting with the more entangled structures from unsupervised techniques. This clear visual separation underscores the advantage of learning a task-specific distance metric that directly optimizes for class discrimination. Visual comparison of original tensors from Galaxy morphology and Crysyal Structure datasets and their reconstructions using deep laerning models, VAE and DEC are shown in Figure~\ref{fig:cry-deep} and~\ref{fig:gal-deep} (see Appendix~\ref{sec:app}), respectively.}

\begin{figure}
	\centering
		\includegraphics[width=\linewidth]{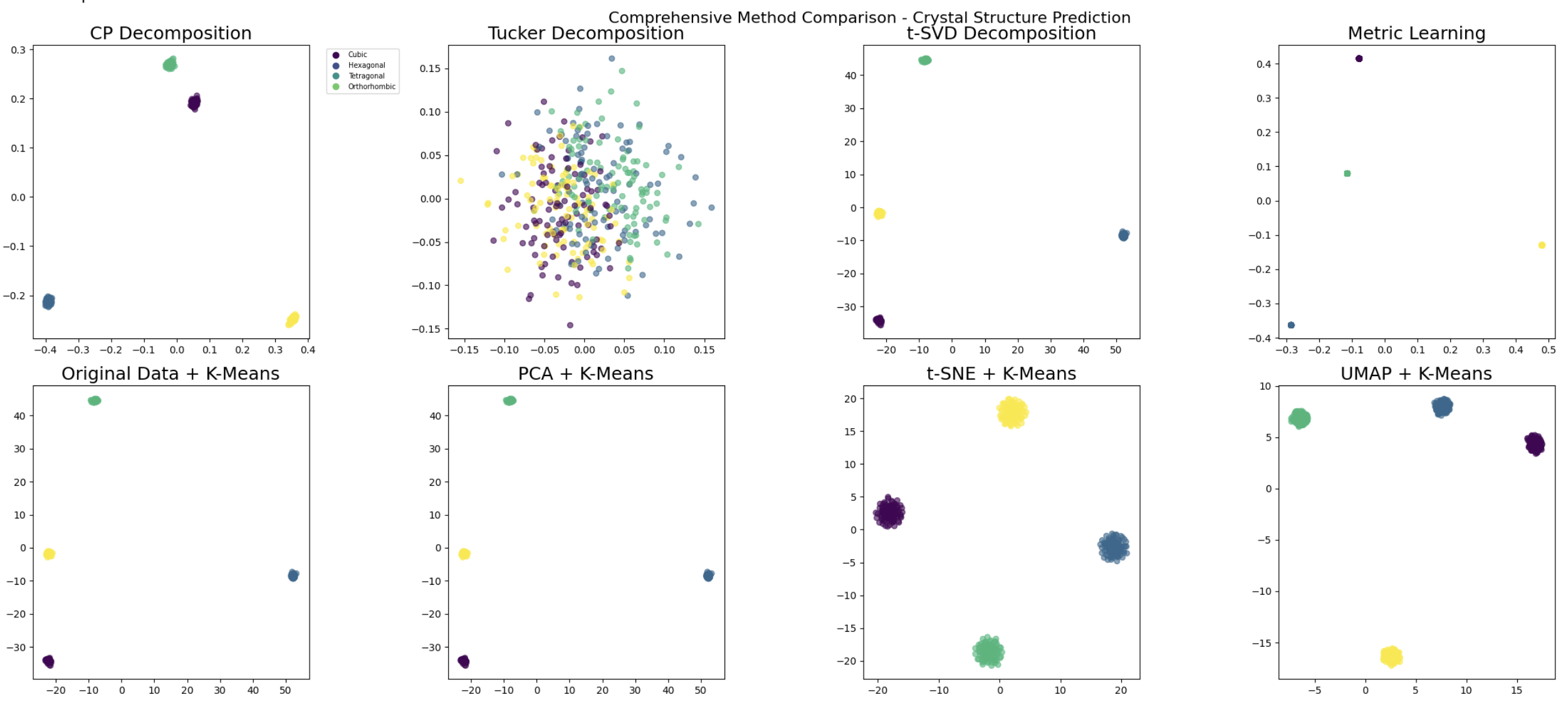}
	\caption{\revised{\textbf{Crystal structure:} Embedding visualizations for Crystal structure dataset: Metric learning produces the most distinct and compact clusters, demonstrating its effectiveness in learning semantically meaningful representations (see also Figure~\ref{fig:clus-gal-sry1}).}}
	\label{fig:clus-gal-sry}
\end{figure}
\begin{figure}[h]
	\centering
	\includegraphics[width=0.9\linewidth]{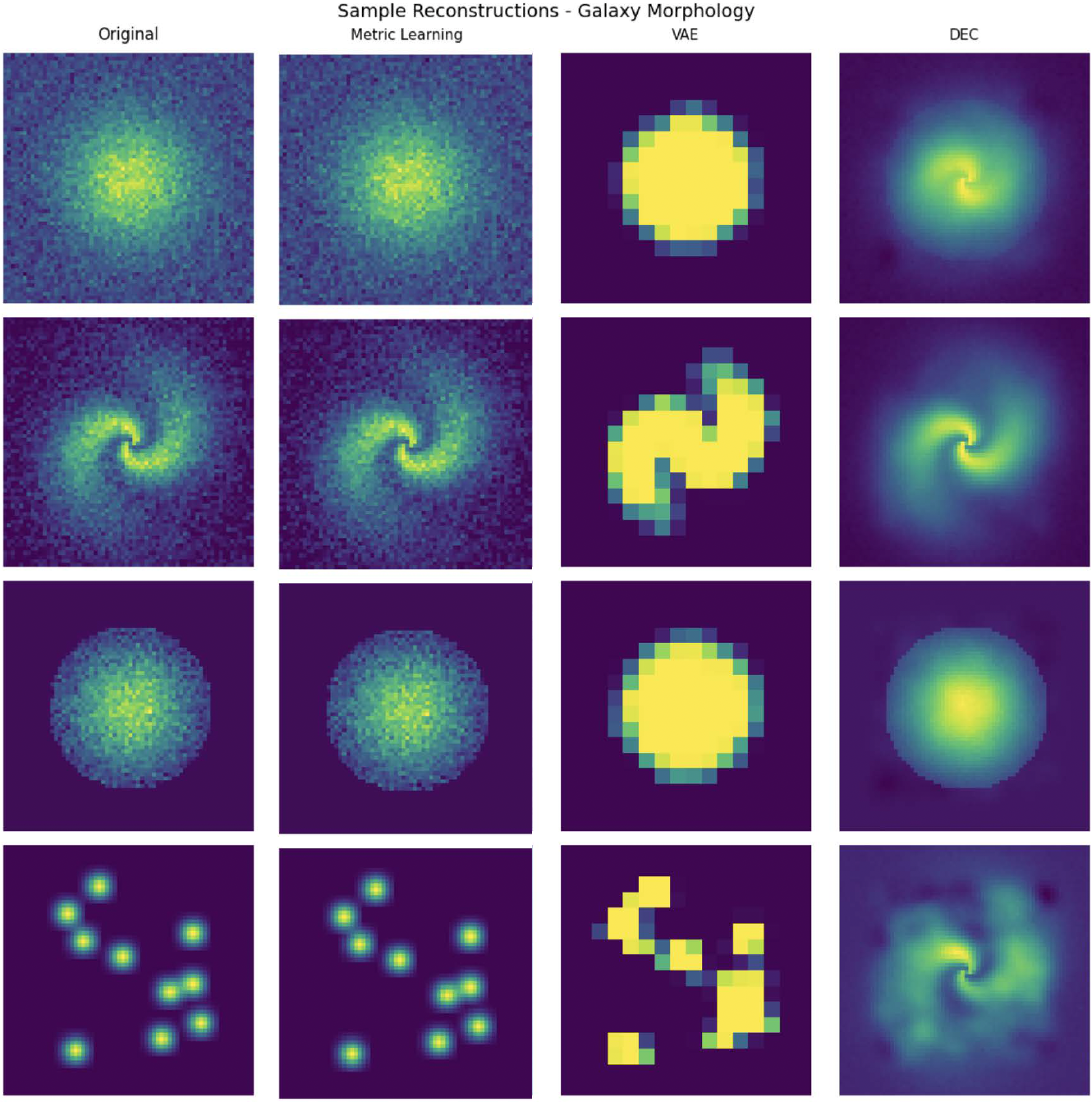}								
	\caption{\revised{\revised{\textbf{Galaxy morphology} }: Visual comparison of original tensor from Galaxy morphology dataset and their reconstructions using deep laerning models, VAE and DEC. All models recived similar pararmerts, number of epochs, normalization and are evaluated on the same metrics (see also Figure~\ref{fig:cry-deep}). }}
	\label{fig:gal-deep}
\end{figure}

These results collectively demonstrate that metric learning excels at capturing the intrinsic categorical structure of simulated data, effectively handling subtle, domain-specific visual differences between classes.

\subsection{Comparison with Tensor Decomposition Methods in Terms of Reconstruction}
In addition to the metric learning approach, we compare against two fundamental tensor decomposition methods: \textit{CP}, \textit{Tucker} \revised {and t-SVD decompositions, as well deep learning models VAE and DEC}. For a third-order tensor $\mathcal{X} \in \mathbb{R}^{I \times J \times K}$ representing our scientific dataset (samples $\times$ height $\times$ width), these decompositions are defined as follows:

A key advantage of the metric learning approach over tensor decomposition methods is its \textit{inherent rank independence}. While tensor decompositions require explicit rank specification, our method automatically learns an optimal latent structure without rank constraints. Both CP and Tucker decompositions (as well as many other low-rank decomposition methods) exhibit strong dependence on the chosen rank $R$. Low ranks may capture insufficient structure, whereas high ranks can lead to overfitting and numerical instability. The metric learning framework, with a latent dimension $d = 64$, learns representations where the \textit{effective rank} is determined by the data complexity rather than pre-specified constraints. Lastly, the convolutional encoder automatically adapts to hierarchical features, capturing both low-rank global structures and high-rank local patterns without explicit rank specification.

\subsubsection{Comparative Analysis Metrics}
We evaluate the decompositions using both reconstruction quality and downstream task performance over the same datasets discussed previously:
\begin{align}
	\text{Reconstruction Error:} \quad & \epsilon = \frac{\|\mathcal{X} - \hat{\mathcal{X}}\|_F}{\|\mathcal{X}\|_F}, \notag \\
	\text{Explained Variance:} \quad & \sigma^2_{\text{explained}} = 1 - \frac{\|\mathcal{X} - \hat{\mathcal{X}}\|_F^2}{\|\mathcal{X}\|_F^2}. \notag 
\end{align}

The fundamental difference lies in how each method handles dimensionality. While tensor decomposition methods require explicit rank specification $R$—typically fixed unless an adaptive approach is implemented \cite{DektorRodgersVenturi2021, sedighin2021adaptive}—they also necessitate rank optimization and multiple rank trials, often guided by theoretical rank estimation. In contrast, the metric learning approach operates with an implicit latent dimension $d$, features an adaptive hierarchical representation, and requires only a single training procedure. This leads to a fully data-driven representation learning process.

A comparison between the metric learning approach and baseline tensor decomposition methods in terms of reconstruction error and explained variance over the same datasets discussed above is provided in Table~\ref{tab:error}. The results for CP, Tucker \revised{and t-SVD} are reported for best-rank approximation for $R=2, 3, 5, 10, 15, 20$. 

\begin{table}
	\centering
	\caption{Comprehensive Performance Comparison of Tensor Methods on Two Datasets}
	\begin{adjustbox}{max width=\textwidth}
		\begin{tabular}{l l c c }
			\toprule
			\textbf{Dataset} & \textbf{Method} & \textbf{Reconstruction Error} & \textbf{Explained Variance} \\
			\midrule
			\multirow{3}{*}{LFW}
			& CP Decomposition     & 0.5300 ($R=20$)&0.7191 \\
			& Tucker Decomposition & 0.4908 ($R=(20, 20, 20)$)&0.7591  \\
			& \revised{t-SVD}&  \revised{0.4908 ($R=20$)}& \revised{0.7590}  \\
			& \revised{VAE}&  \revised{0.6804}& \revised{0.4153}  \\
			& \revised{DEC}&  \revised{0.4835}& \revised{0.5165}  \\
			& \textbf{Metric Learning}      &\textbf{0.0991}&\textbf{0.9901 }  \\
			\midrule
			\multirow{3}{*}{Olivetti }
			& CP Decomposition     &0.6441 ($R=5$)&0.5852   \\
			& Tucker Decomposition & 0.4326 ($R=(20, 20, 20)$)&0.8129  \\
			& \revised{t-SVD}&  \revised{0.4326 ($R=20$)}& \revised{0.8129}  \\
			& \revised{VAE}&  \revised{0.6753}& \revised{0.4190}  \\
			& \revised{DEC}&  \revised{0.4087}& \revised{0.5913}  \\
			& \textbf{Metric Learning }     &\textbf{0.1001 } &\textbf{0.9899}\\
				\midrule
			\multirow{3}{*}{ABIDE}
			& CP Decomposition     & 0.6070 ($R=20$) & 0.6315 \\
			& Tucker Decomposition & 0.4577 ($R=(20, 20, 20)$) & 0.7905  \\
			& \revised{t-SVD}&  \revised{0.4577 ($R=20$)}& \revised{0.7905}  \\
			& \revised{VAE}&  \revised{0.8918}& \revised{0.1122}  \\
			& \revised{DEC}&  \revised{0.5709}& \revised{0.4291}  \\
			& \textbf{Metric Learning}      &\textbf{0.0139} & \textbf{0.9998}\\
			\midrule
			\multirow{3}{*}{Galaxy Morph.}
			& CP Decomposition     & 0.5049 ($R=10$)& 0.7451  \\
			& Tucker Decomposition & 0.4855 ($R=(20, 20, 20)$)& 0.7643  \\
			& \revised{t-SVD}&  \revised{0.7190($R=20$)}& \revised{0.4832}  \\
			& \revised{VAE}&  \revised{0.5572}& \revised{0.4895}  \\
			& \revised{DEC}&  \revised{0.2886}& \revised{0.7114}  \\
			&\textbf{ Metric Learning }     & \textbf{0.0685}  & \textbf{0.9953 } \\
			\midrule
			\multirow{3}{*}{Crystal Struc.}
			& CP Decomposition     & 0.4088 ($R=5$)& 0.8329  \\
			& Tucker Decomposition & 0.3819  ($R=(5, 5, 5)$)& 0.8542\\
			& \revised{t-SVD}&  \revised{0.4440 ($R=5$)}& \revised{0.8026}  \\
			& \revised{VAE}&  \revised{0.8705}& \revised{0.1532}  \\
			& \revised{DEC}&  \revised{0.1851}& \revised{0.8149}  \\
			& \textbf{Metric Learning }     & \textbf{0.0782} & \textbf{0.9938 } \\
			\bottomrule
		\end{tabular}
	\end{adjustbox}
	\label{tab:error}
\end{table}

The reconstruction visualization demonstrates the process of approximating original tensors from their compressed representations obtained through various decomposition methods. For a third-order tensor $\mathcal{X} \in \mathbb{R}^{I \times J \times K}$, each method employs distinct reconstruction mechanisms. 

In {CP decomposition}, the original tensor is reconstructed from rank-1 components through the summation $
\mathcal{\hat{X}} = \sum_{r=1}^{R} \lambda_r \cdot \mathbf{a}_r \circ \mathbf{b}_r \circ \mathbf{c}_r,$
where $\lambda_r$ represents scaling weights and $\mathbf{a}_r, \mathbf{b}_r, \mathbf{c}_r$ are factor matrices capturing different modes of variation. 

The {Tucker decomposition} utilizes a more flexible reconstruction via 
$
\mathcal{\hat{X}} = \mathcal{G} \times_1 \mathbf{A} \times_2 \mathbf{B} \times_3 \mathbf{C},
$
where $\mathcal{G}$ is the core tensor encoding interactions between factors and $\times_n$ denotes the $n$-mode product. 

\revised{The {t-SVD} framework reconstructs the tensor through the tensor product
$
	\mathcal{\hat{X}} = \mathcal{U} * \mathcal{S} * \mathcal{V}^\top,
$
	where $*$ denotes the t-product defined in the Fourier domain, and $\mathcal{S}$ is a f-diagonal tensor containing singular values across frontal slices. This formulation preserves multi-way correlations along the third mode and yields an optimal low-rank approximation under the tensor nuclear norm.} For {metric learning}, reconstruction is achieved through linear regression
$
\mathbf{\hat{X}}_{\text{flat}} = \mathbf{W}\mathbf{Z} + \mathbf{b},
$
where $\mathbf{Z}$ represents the learned embeddings and $\mathbf{W}$ maps these back to the original space. 

\revised{\revised{For the deep learning models, reconstruction is achieved via their respective decoders. For VAE the encoder maps the input tensor into a latent Gaussian distribution, and the decoder reconstructs $\hat{\mathcal{X}}$ by sampling from this distribution and applying non-linear transformations through the network layers. For DEC the encoder produces a latent embedding which is optimized to match a target distribution for clustering, and the reconstruction $\hat{\mathcal{X}}$ is obtained by passing these embeddings through a decoder network trained to minimize reconstruction loss (e.g., mean squared error) while aligning cluster assignments.
		These decoders allow VAE and DEC to approximate the original tensor from learned latent representations, highlighting the difference between reconstruction-focused methods and our metric learning approach, which prioritizes discriminative embeddings over exact reconstruction.}
	}

\revised{The visual comparison reveals that tensor decomposition methods (CP, Tucker, and t-SVD) emphasize structural preservation through explicit algebraic reconstruction, whereas metric learning prioritizes discriminative feature retention over exact reconstruction fidelity.} This distinction reflects their different optimization objectives in capturing brain connectivity patterns. The constructed faces for the Olivetti dataset and cluster visualizations for the LFW dataset are shown in Figure~\ref{fig:rec} and Figure~\ref{fig:rec-lfw}, respectively. \revised{Figure~\ref{fig:abide-deep} shows the functional matrix reconstruction for ABIDE dataset, compared with VAE, and DEC methods. For other compariosn see also Figures~\ref{fig:rec1}--\ref{fig:cry-deep}  in Appendix~\ref{sec:app})}.

\begin{figure}
	\centering
	\includegraphics[width=\linewidth]{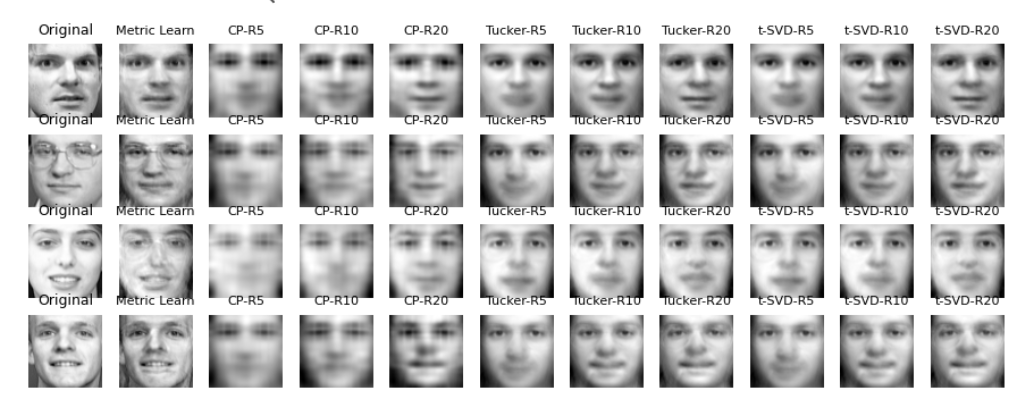}
	\caption{\revised{\revised{\textbf{Olivetti Faces} }: Visual comparison of original tensor from Olivetti face dataset and their reconstructions using different decomposition methods. Here, the results for metric learnign are derived with 50 epochs (see also Figure~\ref{fig:rec1} in Appendix~\ref{sec:app}).} }
	\label{fig:rec}
\end{figure}
\begin{figure}[h]
	\centering
	\includegraphics[width=0.95\linewidth]{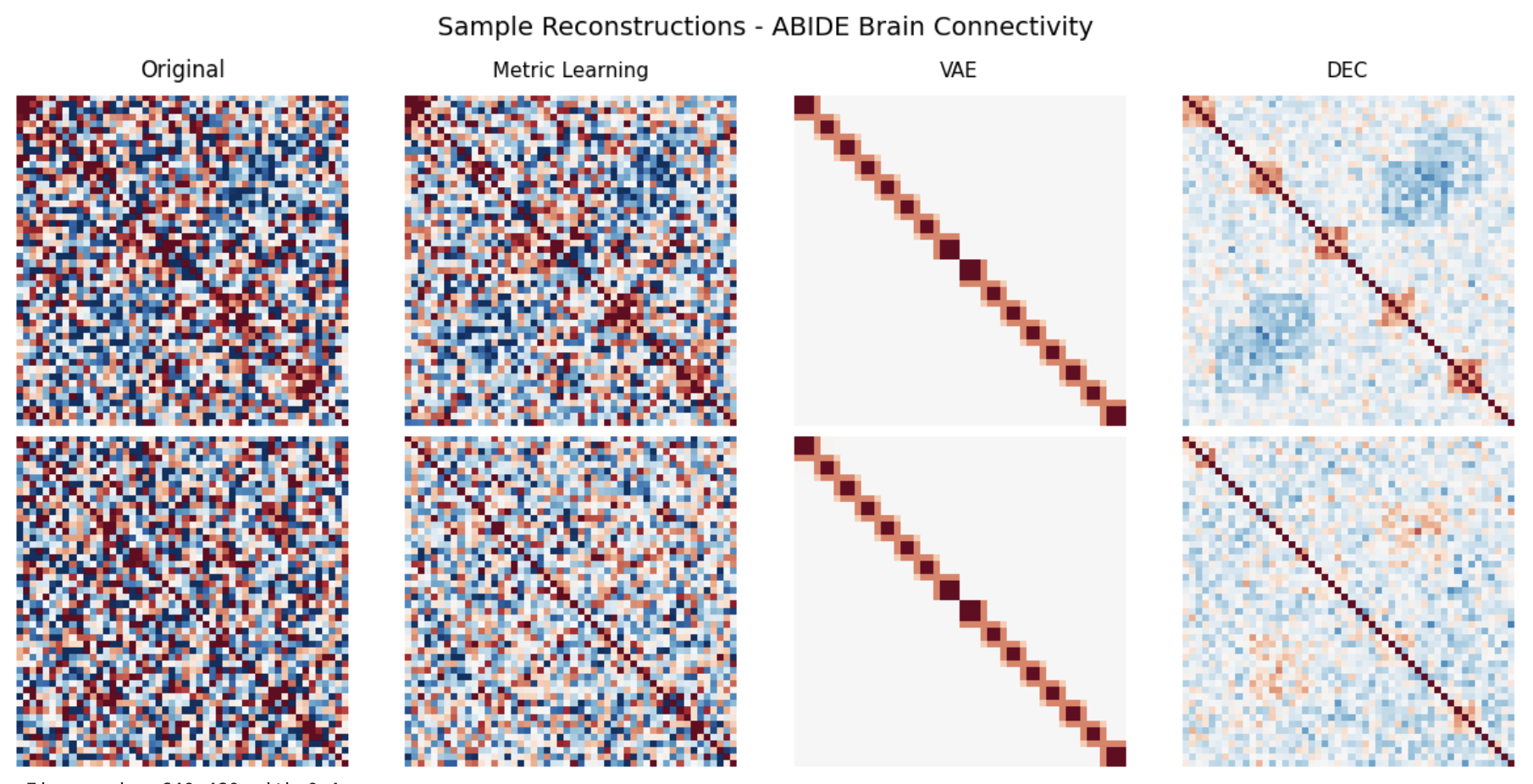}								
	\caption{\revised{\revised{\textbf{ABIDE dataset} }: Visual comparison of original tensor from functional connectivity matrices of ABIDE dataset and their reconstructions using deep laerning models, VAE and DEC. All models recived similar pararmerts, number of epochs, normalization and are evaluated on the same metrics. }}
	\label{fig:abide-deep}
\end{figure}

\begin{figure}[h]
	\centering
	\includegraphics[width=0.9\linewidth]{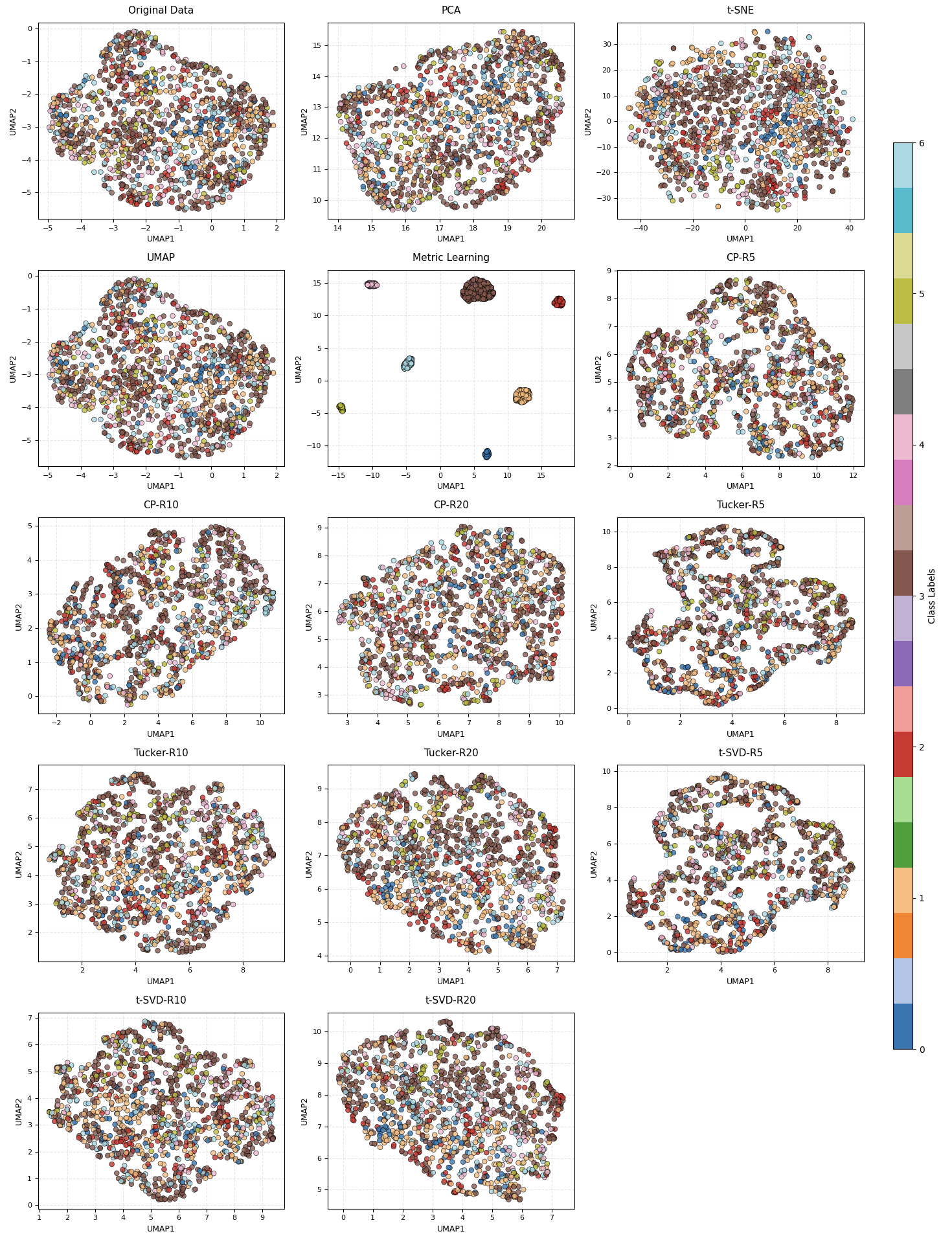}
	\caption{\revised{\textbf{LFW Dataset} }: Two-dimensional manifold visualizations of face embeddings from the LFW dataset generated by different dimensionality reduction and tensor decomposition methods. Each subplot shows the embedding space where points represent individual face images colored by identity, demonstrating the clustering performance and separability achieved by each method \revised{(similar figure for Olivetti Faces dataset is added to Appendix~\ref{sec:app}, Figure~\ref{fig:t-svd-oli})}. }
	\label{fig:rec-lfw}
\end{figure}

Figure~ \ref{fig:abide-ml} demonstrates the two-dimensional manifold visualizations of brain embeddings, comparing original data with CP decomposition, Tucker \revised{ and t-SVD decompositions}, and metric learning approaches. All methods successfully preserve the symmetric structure and hub connectivity patterns essential for neurological analysis.

\begin{figure}[h]
	\centering
				\includegraphics[width=0.295\linewidth]{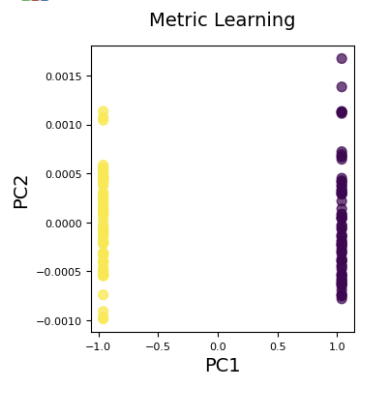}
						\includegraphics[width=0.3\linewidth]{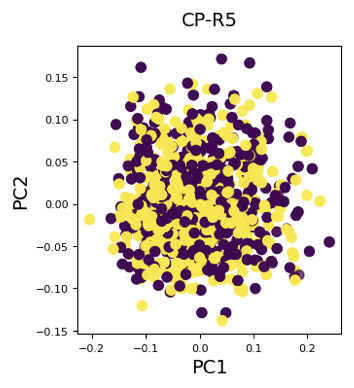}
											\includegraphics[width=0.3\linewidth]{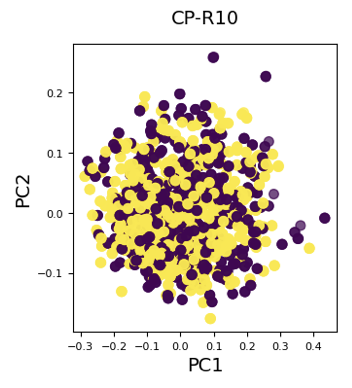}\\
								\includegraphics[width=0.3\linewidth]{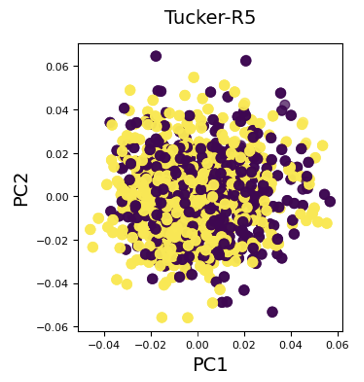}
									\includegraphics[width=0.3\linewidth]{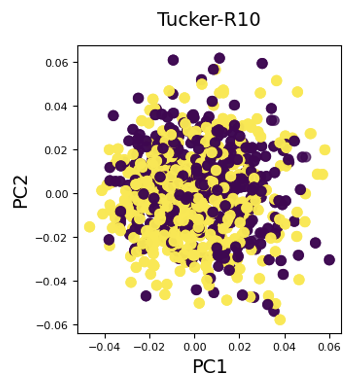}
									\includegraphics[width=0.3\linewidth]{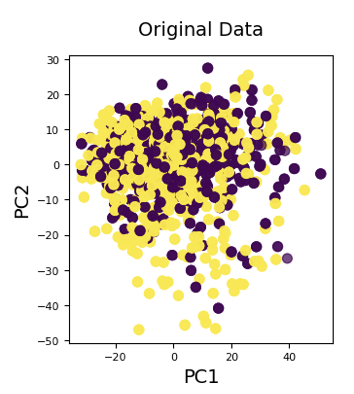}	\\
									\includegraphics[width=0.9\linewidth]{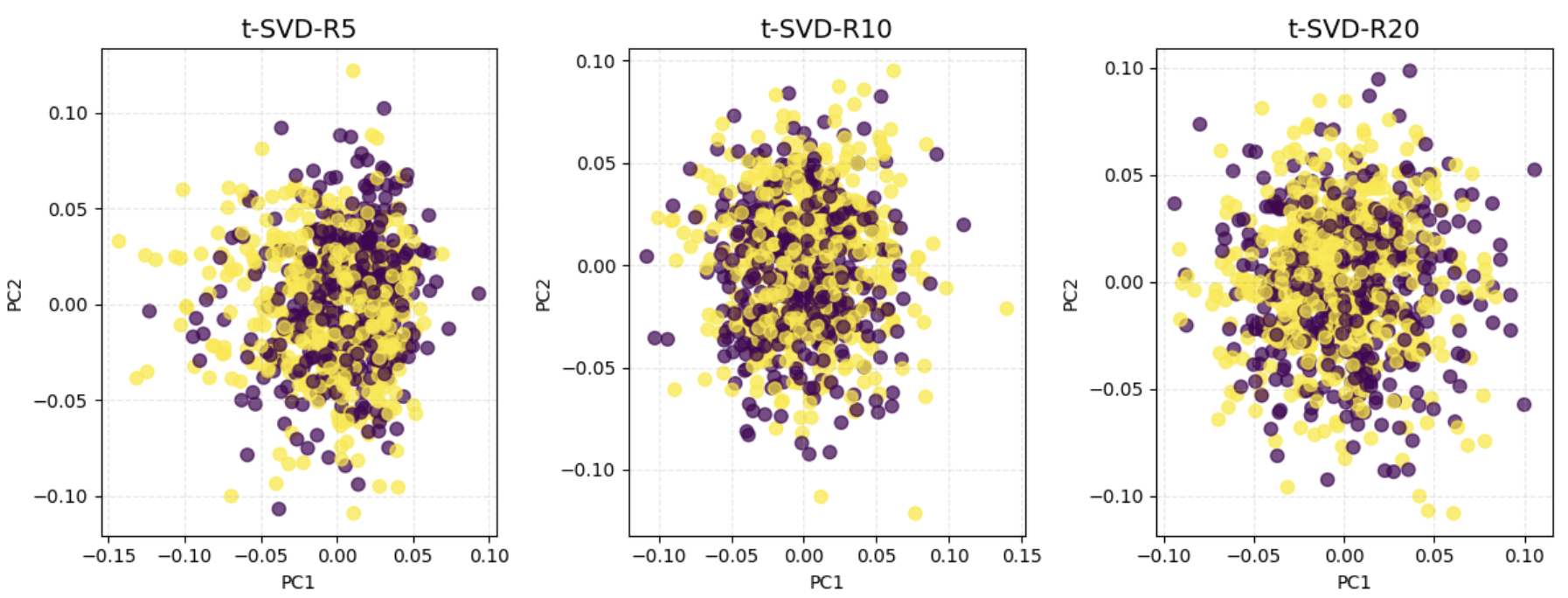}								
	\caption{\revised{\textbf{ABIDE dataset} }: Two-dimensional manifold visualizations of brain embeddings from the ABIDE dataset generated by different dimensionality reduction and tensor decomposition methods. Each subplot shows the embedding space where points represent individual face images colored by identity, demonstrating the clustering performance and separability achieved by each method. The best reported model accuracy for CP with $R=10$ is 0.6000, for \revised {Tucker and t-SVD decompostion with $R=(20, 20, 20)$ and $R=20$ respectively, is around 0.7000,} and for Metric learning is 0.9900.  }
	\label{fig:abide-ml}
\end{figure}

\subsubsection{Empirical Rank Robustness}

Our experiments demonstrate that while tensor decomposition performance varies with rank choice, the metric learning approach maintains consistent performance across different dataset complexities. This robustness stems from non-linear transformations that can capture complex interactions without explicit high-rank decomposition, hierarchical feature learning through convolutional layers that automatically organize features by complexity and task-oriented optimization where the representation is learned specifically for the scientific classification task.

This rank independence makes our approach particularly valuable for scientific datasets where the intrinsic dimensionality may be unknown or vary across different regions of the data space.

\subsection{Comparative Analysis: Metric Learning vs. Transformer Architectures}
\label{sec:comparison}

This section presents a comprehensive comparative analysis between the proposed Metric Learning method and Transformer-based models across five diverse datasets spanning brain connectivity, face recognition, and scientific image classification. \textit{The objective is not to claim universal superiority but  to delineate the specific scenarios, particularly data-constrained environments, where our method offers distinct advantages over Transformer architectures \cite{vaswani2017attention}.} Experiments were conducted on the same datasets as above, namely ABIDE ($111 \times 111$ connectivity matrices), LFW faces ($50 \times 37$), Olivetti faces ($64 \times 64$), Galaxy morphology ($64 \times 64$), and Crystal structure ($64 \times 64$), demonstrating the generalizability of our findings across multiple domains.

\subsubsection{Limitations of Transformer Models}

Our experiments reveal limitations of Transformer architectures across all five datasets.
The standard Transformer's requirement for fixed-length input sequences proved incompatible with small-scale training. For ABIDE I ($111 \times 111 = 12,321$ features), LFW ($50 \times 37 = 1,850$ features), and image datasets ($64 \times 64 = 4,096$ features), the flattened sequence lengths caused consistent failures when batch sizes were smaller than feature dimensions.
 The self-attention mechanism, designed for sequence modeling, proves inefficient for small batches of high-dimensional data, where the sequence length (flattened matrix/image features) consistently exceeds practical batch sizes in data-scarce domains. Thsi was also confrimed through a data efficiency analysis.

\subsubsection{Advantages of the Proposed Metric Learning Framework}

The proposed Metric Learning method demonstrated consistent operational success and competitive performance across all datasets, through geometric learning principles.
The proposde method executed successfully on all datasets across all experimental conditions.
  The most advantage emerged in small-data regimes where Transformers remained inapplicable.

  While traditional methods (Random Forest, SVM) often achieved high performance ($100\%$ on several datasets), metric learning provided competitive results with the added benefit of learned semantic embeddings. The convolutional encoder and triplet loss framework operated successfully across varying input dimensions ($50\times37$ to $111\times111$) without architectural modifications, proving adaptable to diverse data geometries.

\subsubsection{Data Efficiency Analysis}

The systematic evaluation across small dataset sizes revealed distinct operational boundaries, as summarized in Table \ref{tab:data_efficiency}:

\begin{table}[htbp]
	\centering
	\caption{Data Efficiency Comparison Across Dataset Sizes by Accuracy Scores}
	\label{tab:data_efficiency}
	\begin{tabular}{lcccc}
		\hline
		\textbf{Dataset} & \textbf{Size} & \textbf{Metric Learning} &  \textbf{PCA+SVM}& \textbf{Transformer}  \\
		\hline
		ABIDE  & 64& 100.0\%  & 20.0\%& NA  \\
		& 128 & 100.0\%  & 0.0\%& NA  \\
		& 256 & 95.0\%  & 60.0\% & NA \\
		\hline
		LFW Faces & 64 & 92.2\% & 80.0\% & NA  \\
		& 128 & 91.8\%  & 64.1\% & NA\\
		& 256 & 86.6\%  & 61.0\% & NA \\
		\hline
		Olivetti Faces & 128 & 82.2\% & 0.0\%&NA   \\
		& 256 & 90.1\%  & 70.1\% & NA\\
		\hline
		Galaxy Morphologies & 16 & 100.0\% & 80.0\% & NA \\
		& 64 & 100.0\% & 100.0\% & NA  \\
		& 256 & 100.5\%  & 100.0\% & NA \\
		\hline
		Crystal Structures& 64& 100.0\%  & 100.0\% & NA \\
		& 128 & 100.0\%  & 100.0\%& NA \\
		& 256 & 100.0\%  & 100.0\% & NA \\
		\hline
		\hline
	\end{tabular}
\end{table}

For $n < 1000$ Metric learning provides the most reliable approach, balancing operational success with meaningful performance. While traditional methods (Random Forest, SVM) offer excellent performance with minimal complexity over well-structured data, metric learning provides semantically rich embeddings valuable for downstream tasks.

This work establishes metric learning not as a universal solution, but as a reliable paradigm for the widespread class of problems in computational neuroscience and scientific computing where data scarcity is the norm rather than the exception. The method's consistent operational success across diverse domains, combined with its ability to learn meaningful representations from limited data, positions it as an essential tool in the modern machine learning toolkit for scientific applications.

\section{Conclusion}\label{sec:con}

\subsection*{Advantages of the Metric Learning Framework}\label{sec:ad}
\revisedd{The proposed metric learning framework demonstrates strong performance for tasks that require semantically meaningful embeddings, such as clustering and representation analysis. 
	In contrast to classical dimensionality reduction techniques—including CP, Tucker, and t-SVD—which primarily optimize reconstruction error under fixed-rank constraints, our approach explicitly optimizes similarity structure through task-driven loss functions. 
	This design yields representations that are particularly well-suited for downstream tasks involving discrimination, retrieval, and semantic organization of data.\\
	The flexibility of the framework avoids explicit rank selection, allowing the embedding dimension to adapt to data complexity, while modular strategies such as data augmentation, pre-training, and ensemble modeling further improve robustness and generalization. 
	Although the deep network provides substantial modeling capacity, the observed performance improvements are largely attributable to the carefully designed metric learning objectives that directly encode semantic structure. 
	Overall, these properties make the proposed framework a practical and adaptable tool for applications in which meaningful representation learning is prioritized over exact signal reconstruction.
}

\subsection*{ Future Work}
\label{sec:limitations}

While the metric learning framework demonstrates strong performance, several limitations present opportunities for future improvement.

The approach shows sensitivity to class imbalance, where minority classes may receive insufficient representation during triplet mining. Computational overhead from online triplet mining presents scalability challenges, particularly with large batch sizes. Additionally, performance with extremely large numbers of classes requires further validation, and theoretical generalization bounds warrant deeper investigation.

To address these limitations, we plan to develop class-aware triplet mining strategies for imbalanced data and explore more efficient proxy-based losses to reduce computational costs. Scaling the framework to massively multi-class problems and establishing stronger theoretical foundations for generalization guarantees represent key research priorities. These enhancements would further strengthen the framework's applicability across diverse domains.

\section*{\revised{Code Availability \&} Data Statement}
	\revised{An implementation of the proposed method is available at \url{https://github.com/mbagherian/No-Rank-Tensor-Decomposition.}}\\

The datasets used in the preparation of this manuscript are as follows:
\begin{enumerate}
	\item Publicly available \textit{Labeled Faces in the Wild (LFW)} dataset \url{https://www.kaggle.com/datasets/jessicali9530/lfw-dataset} \cite{huang2008labeled}.
	\item Publicly available \textit{Olivetti Faces} dataset (credit to AT\&T Laboratories Cambridge \url{http://www.cl.cam.ac.uk/research/dtg/attarchive/facedatabase.html} and  \cite{pedregosa2011scikit})
	\item Publicly available \textit{Autism Brain Imaging Data Exchange (ABIDE)} dataset \url{https://fcon_1000.projects.nitrc.org/indi/abide/} \cite{di2017enhancing}.
	\item Synthetically simulated \textit{Galaxy Morphology} dataset.
	\item Synthetically simulated \textit{Crystal Structures} dataset.
\end{enumerate}
\section*{Author Contribution Declarations}
M.B. developed the research idea, performed all analyses and experiments, and prepared the manuscript.

\section*{Funding}
This work was supported by the startup fund (Grant number ASE016) provided by the College of Science and Engineering (CoSE) at Idaho State University.

\bibliographystyle{ieeetr}
\bibliography{ref}

@article{acar2011scalable,
	title={Scalable tensor factorizations for incomplete data},
	author={Acar, Evrim and Dunlavy, Daniel M and Kolda, Tamara G and M{\o}rup, Morten},
	journal={Chemometrics and Intelligent Laboratory Systems},
	volume={106},
	number={1},
	pages={41--56},
	year={2011},
	publisher={Elsevier}
}

@article{bagherian2024tensor,
	title={Tensor denoising via dual Schatten norms},
	author={Bagherian, Maryam},
	journal={Optimization Letters},
	volume={18},
	number={5},
	pages={1285--1301},
	year={2024},
	publisher={Springer}
}

@article{bagherian2023classical,
	title={Classical and quantum compression for edge computing: the ubiquitous data dimensionality reduction},
	author={Bagherian, Maryam and Chehade, Sarah and Whitney, Ben and Passian, Ali},
	journal={Computing},
	volume={105},
	number={7},
	pages={1419--1465},
	year={2023},
	publisher={Springer}
}

@article{bagherian2021coupled,
	title={Coupled matrix--matrix and coupled tensor--matrix completion methods for predicting drug--target interactions},
	author={Bagherian, Maryam and Kim, Renaid B and Jiang, Cheng and Sartor, Maureen A and Derksen, Harm and Najarian, Kayvan},
	journal={Briefings in bioinformatics},
	volume={22},
	number={2},
	pages={2161--2171},
	year={2021},
	publisher={Oxford University Press}
}

@article{bagherian2022bilevel,
	title={A Bilevel Optimization Method for Tensor Recovery Under Metric Learning Constraints},
	author={Bagherian, Maryam and Tarzanagh, Davoud A and Dinov, Ivo and Welch, Joshua D},
	journal={arXiv preprint arXiv:2209.00545},
	year={2022}
}

@article{belkin2003laplacian,
	title={Laplacian eigenmaps for dimensionality reduction and data representation},
	author={Belkin, Mikhail and Niyogi, Partha},
	journal={Neural computation},
	volume={15},
	number={6},
	pages={1373--1396},
	year={2003},
	publisher={MIT Press}
}

@article{bengio2013representation,
	title={Representation learning: A review and new perspectives},
	author={Bengio, Yoshua and Courville, Aaron and Vincent, Pascal},
	journal={IEEE transactions on pattern analysis and machine intelligence},
	volume={35},
	number={8},
	pages={1798--1828},
	year={2013},
	publisher={IEEE}
}

@article{bottou2018optimization,
	title={Optimization methods for large-scale machine learning},
	author={Bottou, L{\'e}on and Curtis, Frank E and Nocedal, Jorge},
	journal={SIAM review},
	volume={60},
	number={2},
	pages={223--311},
	year={2018},
	publisher={SIAM}
}

@inproceedings{chen2020simclr_medical,
	title={A simple framework for contrastive learning of visual representations},
	author={Chen, Ting and Kornblith, Simon and Norouzi, Mohammad and Hinton, Geoffrey},
	booktitle={International conference on machine learning},
	pages={1597--1607},
	year={2020},
	organization={PmLR}
}

@article{cichocki2009nonnegative,
	title={Nonnegative matrix and tensor factorization [lecture notes]},
	author={Cichocki, Andrzej and Zdunek, Rafal and Amari, Shun-ichi},
	journal={IEEE signal processing magazine},
	volume={25},
	number={1},
	pages={142--145},
	year={2007},
	publisher={IEEE}
}

@article{DektorRodgersVenturi2021,
	author    = {Alec Dektor and Abram Rodgers and Daniele Venturi},
	title     = {Rank-Adaptive Tensor Methods for High-Dimensional Nonlinear PDEs},
	journal   = {Journal of Scientific Computing},
	year      = {2021},
	volume    = {88},
	number    = {2},
	pages     = {36},
	doi       = {10.1007/s10915-021-01539-3}
}

@article{di2017enhancing,
	title={Enhancing studies of the connectome in autism using the autism brain imaging data exchange II},
	author={Di Martino, Adriana and others},
	journal={Scientific Data},
	volume={4},
	number={170010},
	year={2017},
	publisher={Nature Publishing Group}
}

@article{dieleman2015rotation,
	title={Rotation-invariant convolutional neural networks for galaxy morphology prediction},
	author={Dieleman, Sander and Willett, Kyle W and Dambre, Joni},
	journal={Monthly Notices of the Royal Astronomical Society},
	volume={450},
	number={2},
	pages={1441--1459},
	year={2015},
	publisher={Oxford University Press}
}

@article{faghri2018vse,
	title={Vse++: Improving visual-semantic embeddings with hard negatives},
	author={Faghri, Fartash and Fleet, David J and Kiros, Jamie Ryan and Fidler, Sanja},
	journal={arXiv preprint arXiv:1707.05612},
	year={2017}
}

@article{doi:10.1080/14786440109462720,
	author = { Karl   Pearson   F.R.S. },
	title = {LIII. On lines and planes of closest fit to systems of points in space},
	journal = {The London, Edinburgh, and Dublin Philosophical Magazine and Journal of Science},
	volume = {2},
	number = {11},
	pages = {559-572},
	year  = {1901},
	publisher = {Taylor & Francis},
	doi = {10.1080/14786440109462720},
}

@inproceedings{hadsell2006dimensionality,
	title={Dimensionality reduction by learning an invariant mapping},
	author={Hadsell, Raia and Chopra, Sumit and LeCun, Yann},
	booktitle={2006 IEEE Computer Society Conference on Computer Vision and Pattern Recognition (CVPR'06)},
	volume={2},
	pages={1735--1742},
	year={2006},
	organization={IEEE}
}

@article{hein2007graph,
	title={Graph Laplacians and Their Convergence on Random Neighborhood Graphs},
	author={Hein, Matthias and Audibert, Jean‑Yves and Luxburg, Ulrike von},
	journal={Journal of Machine Learning Research},
	year={2007}
}

@inproceedings{hermans2017defense,
	title={In Defense of the Triplet Loss for Person Re-Identification},
	author={Hermans, Alexander and Beyer, Lucas and Leibe, Bastian},
	booktitle={arXiv preprint arXiv:1703.07737},
	year={2017}
}

@article{hillar2013most,
	title={Most tensor problems are NP-hard},
	author={Hillar, Christopher J and Lim, Lek-Heng},
	journal={Journal of the ACM (JACM)},
	volume={60},
	number={6},
	pages={1--39},
	year={2013},
	publisher={ACM New York, NY, USA}
}

@article{hitchcock1927expression,
	title={The expression of a tensor or a polyadic as a sum of products},
	author={Hitchcock, Frank L},
	journal={Journal of Mathematics and Physics},
	volume={6},
	number={1-4},
	pages={164--189},
	year={1927},
	publisher={Wiley Online Library}
}

@article{hornik1989multilayer,
	author    = {Kurt Hornik and Maxwell B. Stinchcombe and Halbert White},
	title     = {Multilayer feedforward networks are universal approximators},
	journal   = {Neural Networks},
	volume    = {2},
	number    = {5},
	pages     = {359--366},
	year      = {1989},
	doi       = {10.1016/0893-6080(89)90020-8},
	url       = {https://doi.org/10.1016/0893-6080(89)90020-8}
}

@inproceedings{huang2008labeled,
	title={Labeled faces in the wild: A database forstudying face recognition in unconstrained environments},
	author={Huang, Gary B and Mattar, Marwan and Berg, Tamara and Learned-Miller, Eric},
	booktitle={Workshop on faces in'Real-Life'Images: detection, alignment, and recognition},
	year={2008}
}

@article{hubble1926extragalactic,
	title={Extragalactic nebulae.},
	author={Hubble, Edwin P},
	journal={Astrophysical Journal},
	volume={64},
	pages={321--369},
	year={1926}
}

@article{isayev2017universality,
	title={Universal fragment descriptors for predicting properties of inorganic crystals},
	author={Isayev, Olexandr and others},
	journal={Nature communications},
	volume={8},
	number={1},
	pages={15679},
	year={2017},
	publisher={Nature Publishing Group UK London}
}

@article{kingma2013auto,
	title={Auto-encoding variational bayes},
	author={Kingma, Diederik P and Welling, Max},
	journal={arXiv preprint arXiv:1312.6114},
	year={2013}
}

@article{kolda2009tensor,
	title={Tensor decompositions and applications},
	author={Kolda, Tamara G and Bader, Brett W},
	journal={SIAM review},
	volume={51},
	number={3},
	pages={455--500},
	year={2009},
	publisher={SIAM}
}

@article{kruskal1977three,
	title={Three-way arrays: rank and uniqueness of trilinear decompositions, with application to arithmetic complexity and statistics},
	author={Kruskal, Joseph B},
	journal={Linear algebra and its applications},
	volume={18},
	number={2},
	pages={95--138},
	year={1977},
	publisher={Elsevier}
}

@article{kulis2013metric,
	title={Metric learning: A survey},
	author={Kulis, Brian and others},
	journal={Foundations and Trends{\textregistered} in Machine Learning},
	volume={5},
	number={4},
	pages={287--364},
	year={2013},
	publisher={Now Publishers, Inc.}
}

@inproceedings{kulkarni2019canonical,
	title={Canonical surface mapping via geometric cycle consistency},
	author={Kulkarni, Nilesh and Gupta, Abhinav and Tulsiani, Shubham},
	booktitle={Proceedings of the IEEE/CVF International Conference on Computer Vision},
	pages={2202--2211},
	year={2019}
}

@article{kullback1951information,
	title={On information and sufficiency},
	author={Kullback, Solomon and Leibler, Richard A},
	journal={The annals of mathematical statistics},
	volume={22},
	number={1},
	pages={79--86},
	year={1951},
	publisher={JSTOR}
}

@inproceedings{lee2016gradient,
	title={Gradient descent only converges to minimizers},
	author={Lee, Jason D and Simchowitz, Max and Jordan, Michael I and Recht, Benjamin},
	booktitle={Conference on learning theory},
	pages={1246--1257},
	year={2016},
	organization={PMLR}
}

@article{mcinnes2018umap,
	title={UMAP: Uniform Manifold Approximation and Projection for Dimension Reduction},
	author={McInnes, Leland and Healy, John and Melville, James},
	journal={arXiv preprint arXiv:1802.03426},
	year={2018}
}

@article{mo2022rethinking,
	title={Rethinking Prototypical Contrastive Learning through Alignment, Uniformity and Correlation},
	author={Mo, Shentong and Sun, Zhun and Li, Chao},
	journal={arXiv preprint},
	year={2022}
}

@book{mohri2018foundations,
	title={Foundations of machine learning},
	author={Mohri, Mehryar and Rostamizadeh, Afshin and Talwalkar, Ameet},
	year={2018},
	publisher={MIT press}
}

@article{oord2018cpc,
	title={Representation learning with contrastive predictive coding},
	author={Oord, Aaron van den and Li, Yazhe and Vinyals, Oriol},
	journal={arXiv preprint arXiv:1807.03748},
	year={2018}
}

@article{oseledets2011tensor,
	title={Tensor-train decomposition},
	author={Oseledets, Ivan V},
	journal={SIAM Journal on Scientific Computing},
	volume={33},
	number={5},
	pages={2295--2317},
	year={2011},
	publisher={SIAM}
}

@article{pedregosa2011scikit,
	author  = {Pedregosa, Fabian and Varoquaux, Gael and Gramfort, Alexandre and Michel, Vincent and Thirion, Bertrand and Grisel, Olivier and Blondel, Mathieu and Prettenhofer, Peter and Weiss, Ron and Dubourg, Vincent and Vanderplas, Jake and Passos, Alexandre and Cournapeau, David and Brucher, Matthieu and Perrot, Matthieu and Duchesnay, Edouard},
	title   = {Scikit-learn: Machine Learning in Python},
	journal = {Journal of Machine Learning Research},
	volume  = {12},
	pages   = {2825--2830},
	year    = {2011},
	note    = {Dataset: Olivetti Faces, available at \url{https://scikit-learn.org/stable/modules/generated/sklearn.datasets.fetch_olivetti_faces.html}}
}

@article{robbins1951stochastic,
	title={A stochastic approximation method},
	author={Robbins, Herbert and Monro, Sutton},
	journal={The annals of mathematical statistics},
	pages={400--407},
	year={1951},
	publisher={JSTOR}
}

@article{roweis2000nonlinear,
	title={Nonlinear dimensionality reduction by locally linear embedding},
	author={Roweis, Sam T and Saul, Lawrence K},
	journal={Science},
	volume={290},
	number={5500},
	pages={2323--2326},
	year={2000},
	publisher={American Association for the Advancement of Science}
}

@inproceedings{schroff2015facenet,
	title={Facenet: A unified embedding for face recognition and clustering},
	author={Schroff, Florian and Kalenichenko, Dmitry and Philbin, James},
	booktitle={Proceedings of the IEEE conference on computer vision and pattern recognition},
	pages={815--823},
	year={2015}
}

@article{sedighin2021adaptive,
	title={Adaptive rank selection for tensor ring decomposition},
	author={Sedighin, Farnaz and Cichocki, Andrzej and Phan, Anh-Huy},
	journal={IEEE Journal of Selected Topics in Signal Processing},
	volume={15},
	number={3},
	pages={454--463},
	year={2021},
	publisher={IEEE}
}

@book{shalev2014understanding,
	title={Understanding machine learning: From theory to algorithms},
	author={Shalev-Shwartz, Shai and Ben-David, Shai},
	year={2014},
	publisher={Cambridge university press}
}

@inproceedings{tancik2020fourier,
	title={Fourier features let networks learn high frequency functions in low dimensional domains},
	author={Tancik, Matthew and Srinivasan, Pratul P and Mildenhall, Ben and Fridovich-Keil, Sara and Raghavan, Nithin and Singhal, Utkarsh and Ramamoorthi, Ravi and Barron, Jonathan T and Ng, Ren},
	booktitle={Advances in Neural Information Processing Systems},
	volume={33},
	pages={7537--7547},
	year={2020}
}

@article{tenenbaum2000global,
	title={A global geometric framework for nonlinear dimensionality reduction},
	author={Tenenbaum, Joshua B and De Silva, Vin and Langford, John C},
	journal={science},
	volume={290},
	number={5500},
	pages={2319--2323},
	year={2000},
	publisher={American Association for the Advancement of Science}
}

@article{tucker1966some,
	title={Some mathematical notes on three-mode factor analysis},
	author={Tucker, Ledyard R},
	journal={Psychometrika},
	volume={31},
	number={3},
	pages={279--311},
	year={1966},
	publisher={Springer}
}

@article{maaten2008visualizing,
	title={Visualizing data using t-SNE},
	author={Van der Maaten, Laurens and Hinton, Geoffrey},
	journal={Journal of Machine Learning Research},
	volume={9},
	number={11},
	year={2008}
}

@article{vaswani2017attention,
	title={Attention is all you need},
	author={Vaswani, Ashish and Shazeer, Noam and Parmar, Niki and Uszkoreit, Jakob and Jones, Llion and Gomez, Aidan N and Kaiser, {\L}ukasz and Polosukhin, Illia},
	journal={Advances in neural information processing systems},
	volume={30},
	year={2017}
}

@article{walmsley2022galaxy,
	title={Galaxy Zoo DECaLS: Detailed visual morphology measurements from volunteers and deep learning for 314,000 galaxies},
	author={Walmsley, Mike and others},
	journal={Monthly Notices of the Royal Astronomical Society},
	volume={509},
	number={3},
	pages={3966--3988},
	year={2022},
	publisher={Oxford University Press}
}

@article{wang2023transformed,
	title={Transformed low-rank parameterization can help robust generalization for tensor neural networks},
	author={Wang, Andong and Li, Chao and Bai, Mingyuan and Jin, Zhong and Zhou, Guoxu and Zhao, Qibin},
	journal={Advances in Neural Information Processing Systems},
	volume={36},
	pages={3032--3082},
	year={2023}
}

@inproceedings{wangtowards,
	title={Towards a Geometric Understanding of Tensor Learning via the t-Product},
	author={Wang, Andong and Qiu, Yuning and Huang, Haonan and Jin, Zhong and Zhou, Guoxu and Zhao, Qibin},
	booktitle={The Thirty-ninth Annual Conference on Neural Information Processing Systems},
	year={2025}
}

@inproceedings{wang2018cosface,
	title={Cosface: Large margin cosine loss for deep face recognition},
	author={Wang, Hao and Wang, Yitong and Zhou, Zheng and Ji, Xing and Gong, Dihong and Zhou, Jingchao and Li, Zhifeng and Liu, Wei},
	booktitle={Proceedings of the IEEE conference on computer vision and pattern recognition},
	pages={5265--5274},
	year={2018}
}

@article{wang2020understanding,
	title={Understanding contrastive representation learning through alignment and uniformity on the hypersphere},
	author={Wang, Tongzhou and Isola, Phillip},
	booktitle={International Conference on Machine Learning},
	pages={9929--9939},
	year={2020},
	organization={PMLR}
}

@inproceedings{wang2019multi,
	title={Multi-similarity loss with general pair weighting for deep metric learning},
	author={Wang, Xun and Han, Xintong and Huang, Weilin and Dong, Dengke and Scott, Matthew R},
	booktitle={Proceedings of the IEEE/CVF conference on computer vision and pattern recognition},
	pages={5022--5030},
	year={2019}
}

@article{weinberger2005distance,
	title={Distance metric learning for large margin nearest neighbor classification},
	author={Weinberger, Kilian Q and Blitzer, John and Saul, Lawrence},
	journal={Advances in neural information processing systems},
	volume={18},
	year={2005}
}

@inproceedings{wu2017sampling,
	title={Sampling matters in deep embedding learning},
	author={Wu, Chao-Yuan and Manmatha, R and Smola, Alexander J and Krahenbuhl, Philipp},
	booktitle={Proceedings of the IEEE international conference on computer vision},
	pages={2840--2848},
	year={2017}
}

@inproceedings{xie2016unsupervised,
	title={Unsupervised deep embedding for clustering analysis},
	author={Xie, Junyuan and Girshick, Ross and Farhadi, Ali},
	booktitle={International conference on machine learning},
	pages={478--487},
	year={2016},
	organization={PMLR}
}

@article{xie2018crystal,
	title={Crystal graph convolutional neural networks for an accurate and interpretable prediction of material properties},
	author={Xie, Tian and Grossman, Jeffrey C},
	journal={Physical review letters},
	volume={120},
	number={14},
	pages={145301},
	year={2018},
	publisher={APS}
}

@article{xue2024tensor,
	title={Tensor convolution-like low-rank dictionary for high-dimensional image representation},
	author={Xue, Jize and Zhao, Yongqiang and Wu, Tongle and Chan, Jonathan Cheung-Wai},
	journal={IEEE Transactions on Circuits and Systems for Video Technology},
	year={2024},
	publisher={IEEE}
}
\newpage 
\section{Appendix}\label{sec:app}
\begin{figure}[h]
	\centering
	\includegraphics[width=\linewidth]{fig/21-model_F}
	\caption{\revised{\textbf{Model Overview}: Detailed overview of the proposed no-rank metric learning framework for tensor data. Each mode-$n$ slice of the tensor is mapped through a neural network to an embedding space. Triplet-based losses (anchor, positive, negative) and regularization terms (diversity, uniformity, locality preservation) shape a semantically meaningful embedding without requiring explicit rank constraints.}}
	\label{fig:modelfull}
\end{figure}

\begin{figure}[h]
	\centering
	\includegraphics[width=0.9\linewidth]{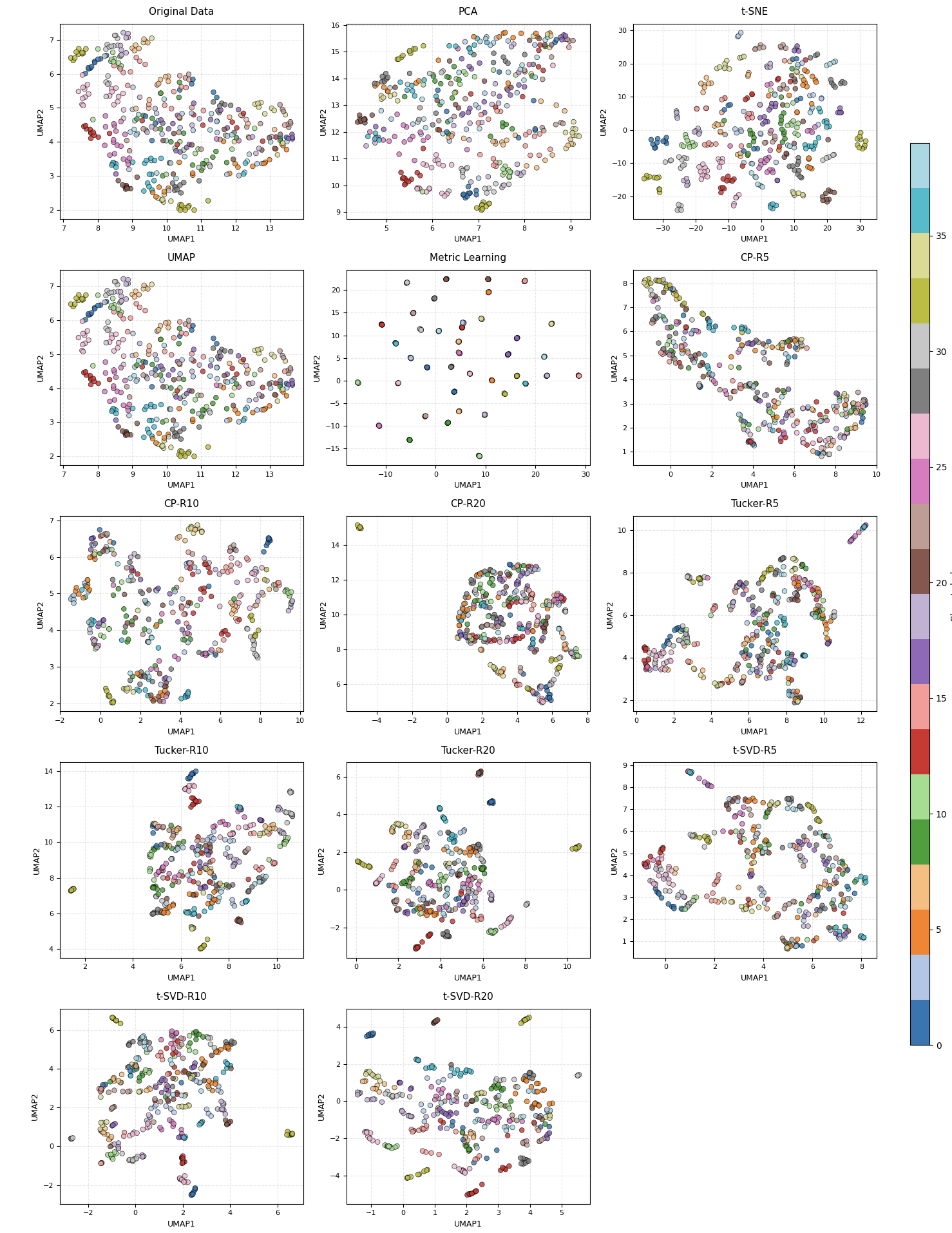}
	\caption{\revised{\textbf{Olivetti Faces: }Two-dimensional manifold visualizations of face embeddings from the Olivetti Faces dataset generated by different dimensionality reduction and tensor decomposition methods. Each subplot shows the embedding space where points represent individual face images colored by identity, demonstrating the clustering performance and separability achieved by each method. }}
	\label{fig:t-svd-oli}
\end{figure}

\begin{figure}
	\centering
	\includegraphics[width=\linewidth]{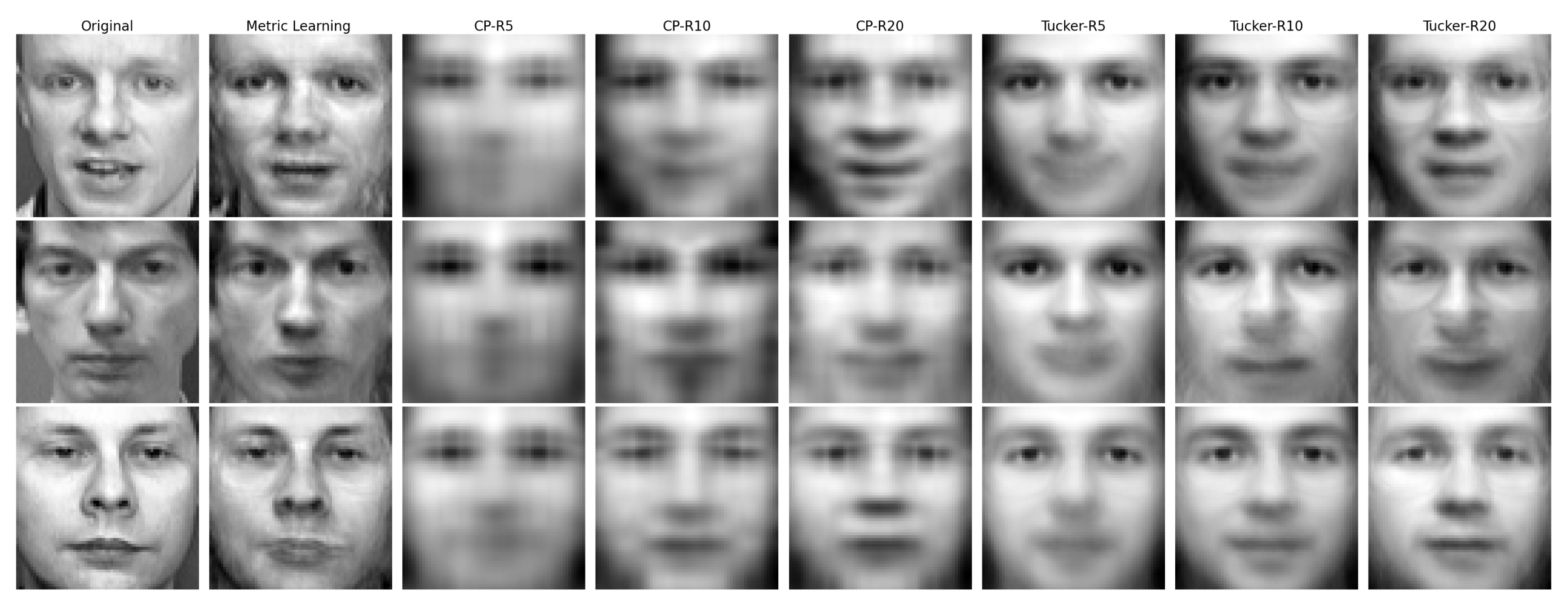}
	\caption{\textbf{Olivetti Faces: } Visual comparison of original tensor from Olivetti face dataset and their reconstructions using different decomposition methods. Here, the results for metric learnign are derived with 50 epochs. }
	\label{fig:rec1}
\end{figure}

\begin{figure}[h]
	\centering
	\includegraphics[width=0.7\linewidth]{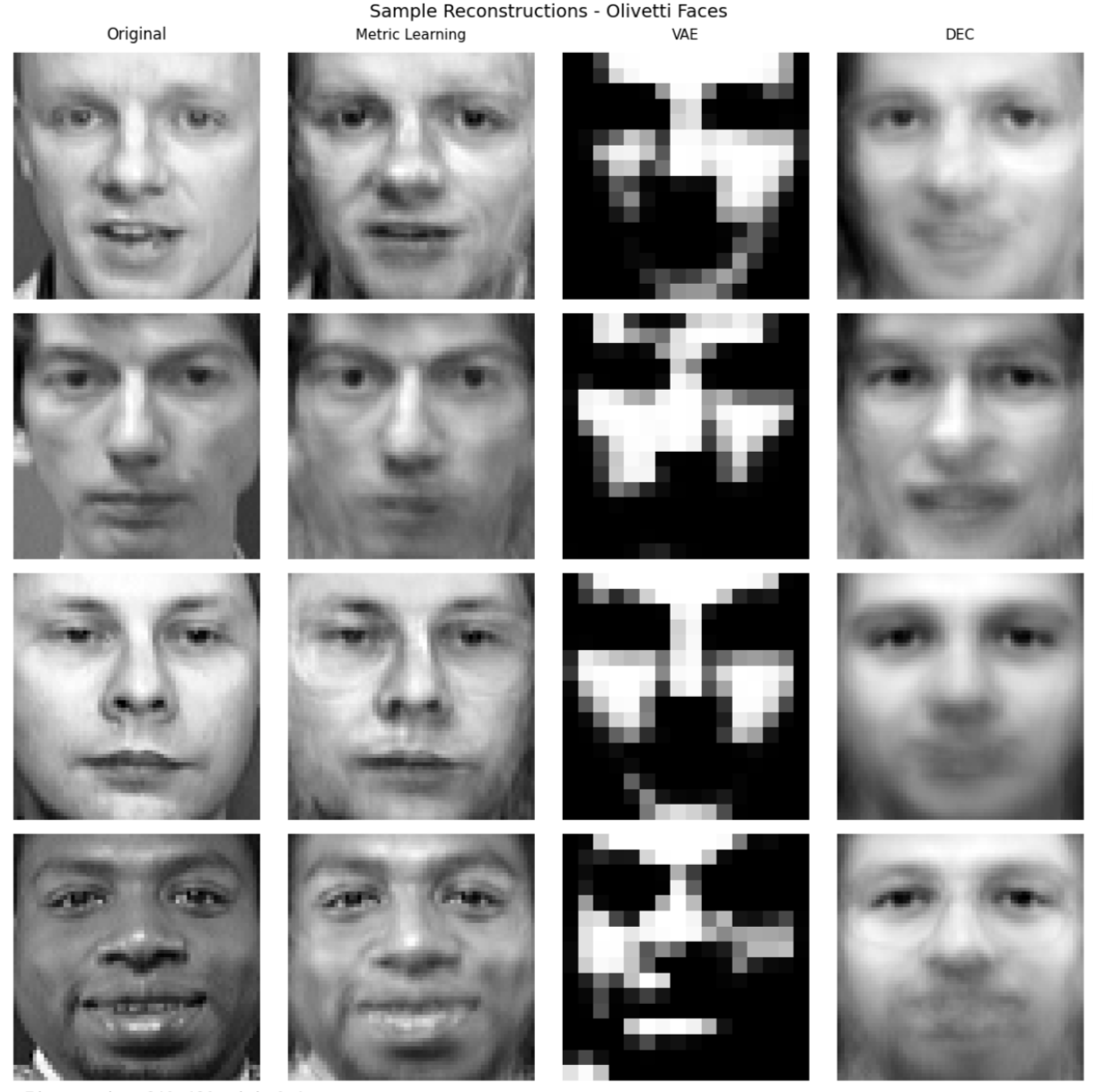}								
	\caption{\revised{\textbf{Olivetti Faces: }Visual comparison of original tensor from Olivetti face dataset and their reconstructions using deep laerning models, VAE and DEC. All models recived similar pararmerts, number of epochs, normalization and are evaluated on the same metrics. }}
	\label{fig:oli-deep}
\end{figure}

\begin{figure}
	\centering
	\includegraphics[width=\linewidth]{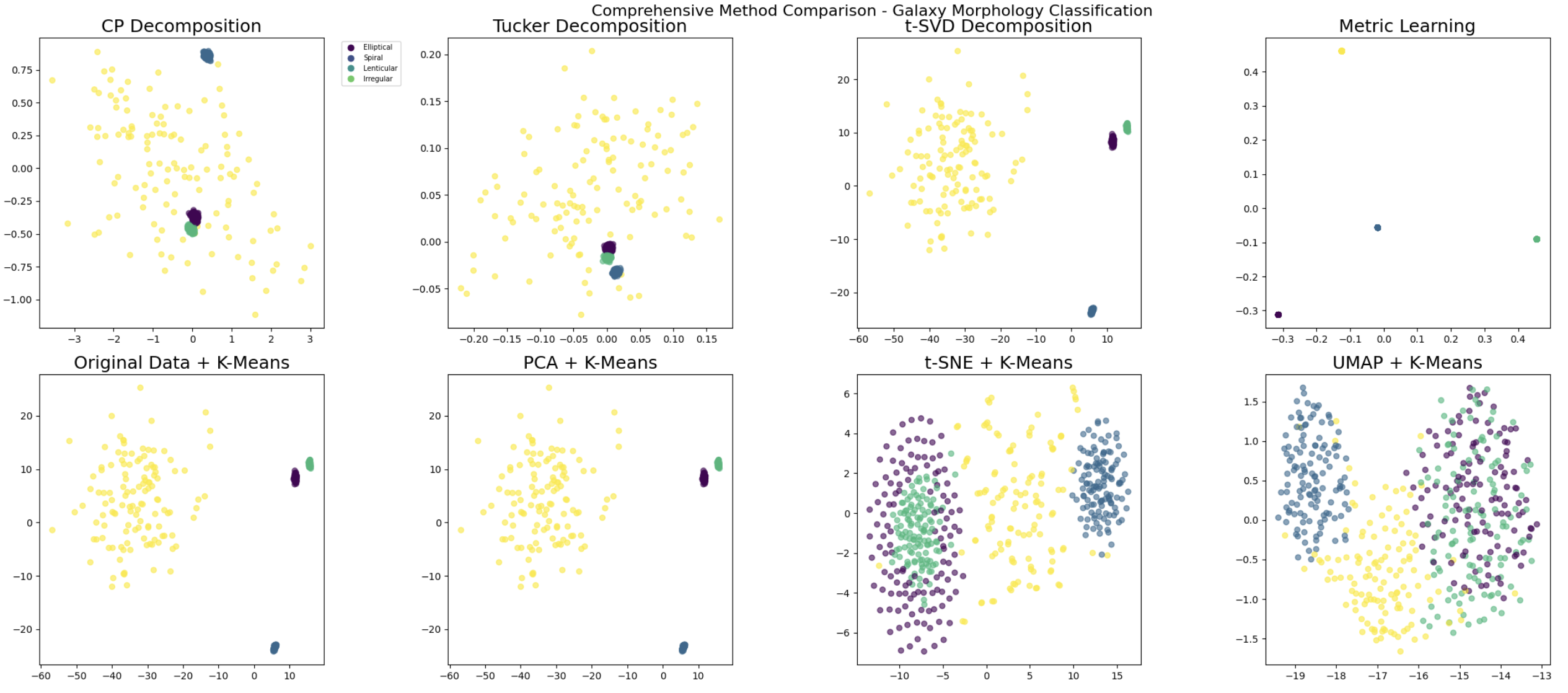}\\
	\caption{\revised{\textbf{Galaxy morphology: } Embedding visualizations for Galaxy morphology datase}: Metric learning produces the most distinct and compact clusters, demonstrating its effectiveness in learning semantically meaningful representations.}
	\label{fig:clus-gal-sry1}
\end{figure}

\begin{figure}[h]
	\centering
	\includegraphics[width=0.7\linewidth]{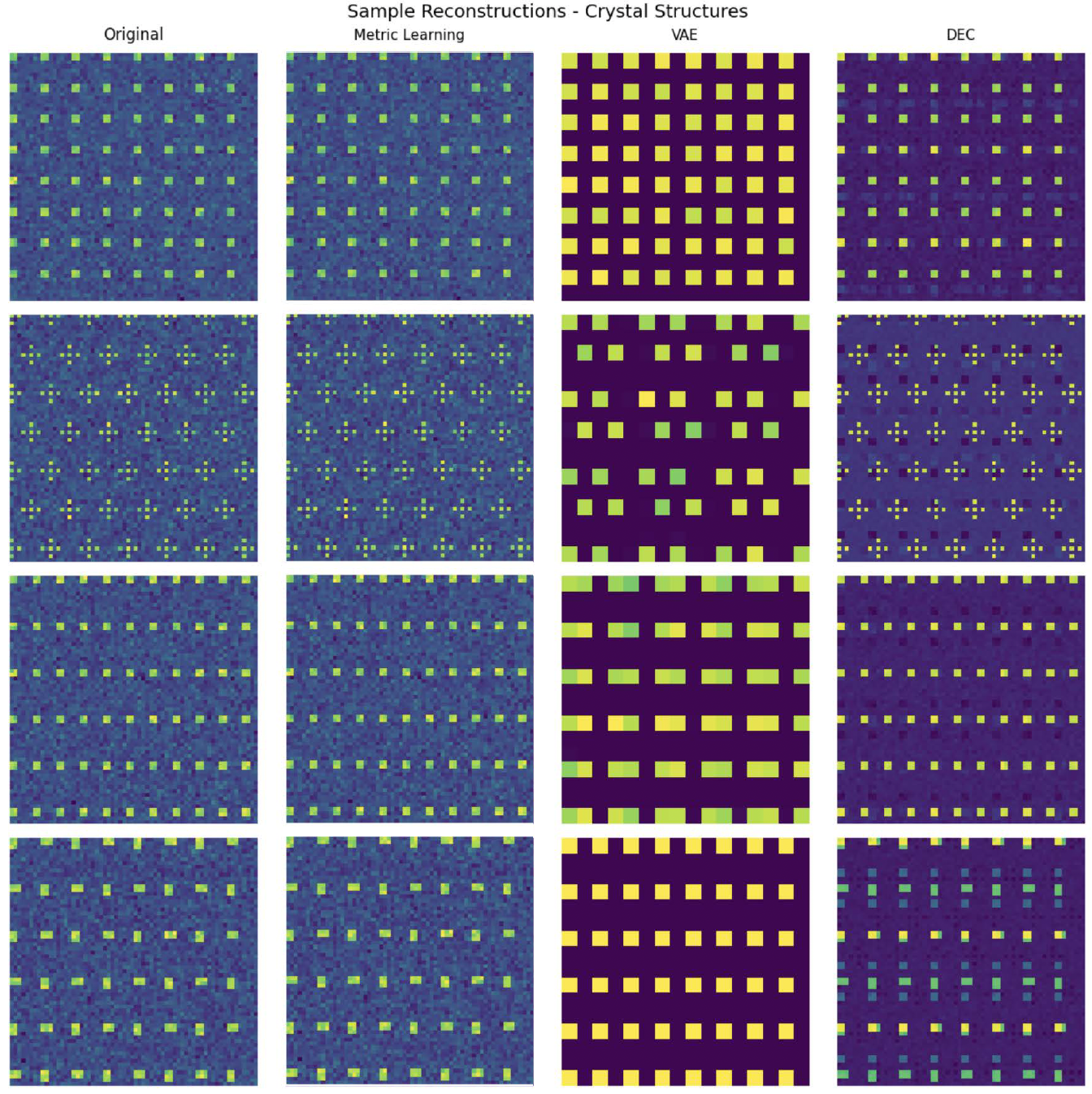}								
	\caption{\revised{\textbf{Crystal structures: }Visual comparison of original tensor from Crystal structure dataset and their reconstructions using deep laerning models, VAE and DEC. All models recived similar pararmerts, number of epochs, normalization and are evaluated on the same metrics. }}
	\label{fig:cry-deep}
\end{figure}

\end{document}